\documentclass{article}

\usepackage{microtype}
\usepackage{graphicx}
\usepackage{subfigure}
\usepackage{booktabs} 

\usepackage{hyperref}

\usepackage[preprint]{icml2026}

\usepackage{amsmath}
\usepackage{amssymb}
\usepackage{mathtools}
\usepackage{amsthm}

\usepackage[capitalize,noabbrev]{cleveref}

\theoremstyle{plain}
\newtheorem{theorem}{Theorem}[section]
\newtheorem{proposition}[theorem]{Proposition}
\newtheorem{lemma}[theorem]{Lemma}
\newtheorem{corollary}[theorem]{Corollary}
\theoremstyle{definition}
\newtheorem{definition}[theorem]{Definition}

\theoremstyle{remark}

\usepackage[textsize=tiny]{todonotes}

\usepackage{acronym}
\newacro{CSP}{Constraint-Satisfaction Problem}
\newacro{SAT}{Boolean Satisfiability}
\newacro{CLS}{Continuous Local Search}
\newacro{CP}{Constraint Programming}
\newacro{CDCL}{Conflict-Driven Clause Learning}
\newacro{ILP}{Integer Linear Programming}
\newacro{CNF}{Conjunctive Normal Form}
\newacro{ML}{Machine Learning}
\newacro{BDD}{Binary Decision Diagram}
\newacro{WFE}{Walsh-Fourier Expansion}
\newacro{COP}{Circuit-Output Probability}
\newacro{PGD}{Projected Gradient Descent}
\newacro{MD}{Mirror Descent}
\newacro{HD}{Hybrid Descent}
\newacro{FISTA}{Fast Iterative Shrinkage-Thresholding Algorithm}

\DeclareMathOperator*{\argmin}{arg\,min}

\usepackage{listings}
\floatstyle{ruled}
\newfloat{listing}{tb}{lst}{}
\floatname{listing}{Listing}
\lstdefinestyle{mystyle}{
    basicstyle=\ttfamily\small,
    keywordstyle=\color{blue}\bfseries,
    commentstyle=\color{gray},
    stringstyle=\color{orange},
    numberstyle=\tiny\color{gray},
    numbers=left,
    numbersep=5pt,
    backgroundcolor=\color{white},
    frame=single,
    breaklines=true,
    tabsize=2,
    captionpos=b,
    showstringspaces=false,
    aboveskip=0.8em,
    belowskip=0.8em,
}

\icmltitlerunning{\textsc{FourierCSP}: Differentiable Constraint Satisfaction Problem Solving by Walsh-Fourier Expansion}

\begin{document}

\twocolumn[
  \icmltitle{\textsc{FourierCSP}: Differentiable Constraint Satisfaction Problem Solving by Walsh-Fourier Expansion}



    \begin{icmlauthorlist}
        \icmlauthor{Yunuo Cen}{nus}
        \icmlauthor{Zixuan Wang}{nus}
        \icmlauthor{Jintao Zhang}{nus}
        \icmlauthor{Zhiwei Zhang}{rice}
        \icmlauthor{Xuanyao Fong}{nus}
    \end{icmlauthorlist}
    
    \icmlaffiliation{nus}{Department of Electrical and Computer Engineering, National University of Singapore, Singapore}
    \icmlaffiliation{rice}{Department of Computer Science, Rice University, Houston, Texas, United States}

    \icmlcorrespondingauthor{Xuanyao Fong}{kelvin.xy.fong@nus.edu.sg}

  \icmlkeywords{Machine Learning, ICML}

  \vskip 0.3in
]



\printAffiliationsAndNotice{}  

\begin{abstract}
    The Constraint-Satisfaction Problem (CSP) is fundamental in mathematics, physics, and theoretical computer science.  
    Continuous Local Search (CLS) solvers, as recent advancements, can achieve highly competitive results on certain classes of Boolean Satisfiability (SAT) problems. 
    Motivated by these advances, we extend the {CLS} framework from Boolean SAT to general {CSP} with finite-domain variables and expressive constraint formulations.
    We present \textsc{FourierCSP}, a continuous optimization framework that generalizes the Walsh-Fourier transform to {CSP}, allowing for transforming versatile constraints to compact multilinear polynomials, thereby avoiding the need for auxiliary variables and memory-intensive encodings.
    We employ projected subgradient and mirror descent algorithms with provable convergence guarantees, and further combine them to accelerate gradient-based optimization.
    Empirical results on benchmark suites demonstrate that \textsc{FourierCSP} is scalable and competitive, significantly broadening the class of problems that can be efficiently solved by differentiable {CLS} techniques and paving the way toward end-to-end neurosymbolic integration.
\end{abstract}

\section{Introduction}
The \ac{CSP} is foundational in mathematics, physics, and theoretical computer science.
Formally, a \ac{CSP} consists of variables defined over finite domains and a set of constraints restricting their values; the objective is to identify an assignment that satisfies all constraints
Applications of \ac{CSP} solving span diverse fields, including artificial intelligence~\cite{baluta2019quantitative}, database query~\cite{kolaitis1998conjunctive}, operational research~\cite{brailsford1999constraint}, statistical physics~\cite{nagy2024ising}, information theory~\cite{golia2022scalable}, and quantum computing~\cite{vardi2022quantum}.
Despite their significant importance, \ac{CSP} inherently exhibits a combinatorial nature and is generally NP-complete, presenting a substantial challenge in solving it efficiently.
Consequently, extensive research and development efforts have been dedicated to advancing \ac{CSP} solving~\cite{schulte2006gecode,hebrardmistral,kuchcinski2013jacop,Prud'homme2022,perron_et_al:LIPIcs.CP.2023.3,krupke2024cpsat}. 

Although classical \ac{CSP} engines~\cite{schulte2006gecode, hebrardmistral, kuchcinski2013jacop, Prud'homme2022} are powerful, their reliance on backtracking and discrete propagation creates a fundamental barrier for deep learning integration.
The lack of differentiability precludes backpropagation, hindering the embedding of constraint reasoning into neural architectures.
Recent hybrid approaches have attempted to bridge this gap by learning heuristics for branching and cut selection~\cite{khalil2016learning, zhang2022learning, ling2024learning, vo2025learning, zeng2025beyond, gasse2019exact}.
However, these methods merely augment the discrete search process and do not yield differentiable solvers.

In parallel, \ac{CLS} approaches for \ac{SAT} have demonstrated that gradient-based optimization on smooth multilinear relaxations can yield highly effective differentiable solvers~\cite{kyrillidis2020fouriersat, kyrillidis2021continuous}.
Yet, these formulations rely on the algebraic properties of Boolean variables $\{\pm1\}$, making them inapplicable to the multi-valued domains and rich constraints found in general \ac{CSP}s.
Existing differentiable \ac{CSP} techniques generally rely on ad-hoc heuristics or specific constraint types~\cite{xu2018towards, goyal2024deepsade, xu2025selfsupervised}, lacking a unified formulation for general finite-domain constraints.

This gap necessitates a differentiable framework that: (i) operates directly on finite-domain variables, avoiding the dimensionality explosion of Boolean encodings; (ii) can formulate versatile constraints; and (iii) ensures the continuous relaxation faithfully captures the structure of the discrete feasibility region.
To achieve this, we utilize the Walsh-Fourier basis, which provides an orthogonal multilinear expansion that yields smooth optimization landscapes, ideal for projected subgradient and mirror-descent methods.
While Fourier-based relaxations have proven effective for Boolean \ac{SAT}~\cite{kyrillidis2020fouriersat, kyrillidis2021continuous}, in this work, we generalize the Walsh-Fourier transform to the finite-domain setting. 
This provides a principled algebraic foundation for differentiable \ac{CSP} solving, extending the capabilities of \ac{CLS} beyond the Boolean domain.
The contribution of this paper can be summarized as follows:

\textbf{A Uniform Differentiable \ac{CSP} Framework}: 
We provide a rigorous algebraic generalization of the Walsh-Fourier transform for variables with finite domains.
The proposed framework is called \textsc{FourierCSP}. 
It established a principled basis for differentiable constraint solving.

\textbf{Provable Optimization Algorithm}:
In addition to the \textsc{FourierCSP} framework, we develop specialized variants of projected subgradient and mirror-descent methods tailored to this setting. 
We provide rigorous convergence guarantees, ensuring the reliability of the continuous search.

\textbf{Efficient Implementation and Empirical Evaluation}: 
We conduct extensive evaluations on standard benchmark suites. Our results demonstrate that \textsc{FourierCSP} is highly scalable and competitive with existing baselines, significantly broadening the scope of problems solvable by differentiable techniques.

\section{Preliminaries}
\paragraph{Constraint Satisfaction Problem}
Let $X=\{x_1, ..., x_n\}$ be a set of variables, and $D=\{\mathrm{dom}(x_1), ..., \mathrm{dom}(x_n)\}$ be a set of finite domains, and $C=\{c_1, ..., c_m\}$ be a set of constraints.
Every constraint $c_i\in{}C$ is described by an atom set $A_i\subseteq{}X$.
The (finite-domain) CSP is to find an assignment of $X$ over $D$ such that a finite-domain formula, \emph{i.e.}, the conjunction of all constraints, $F=\bigwedge_{c_i\in{}C}c_i$, is satisfiable.

\paragraph{Walsh-Fourier Transform}
We define a function with finite domain by $f:D\to\{0,1\}$, where $0$ and $1$ present $\texttt{false}$, $\texttt{true}$ of the finite constraint.

The Walsh-Fourier Transform is a method for transforming a function with a finite domain into a multilinear polynomial.
Given the Fourier bases, any function defined on a finite domain has an equivalent Fourier representation.

\begin{definition}[Fourier Bases]\label{def:bases}
    Given a function on a finite domain $f:D \to \{0,1\}$, we define the Fourier bases $\phi_{\alpha}$ by
    \begin{equation*}
        \phi_{\alpha}(X)=\prod_{i=1}^n\phi_{\alpha_i}(x_i), 
    \end{equation*}
    where $\alpha_i\in\mathrm{dom}(x_i)$, and $\phi_{\alpha_i}$ is an indicator basis, \emph{i.e.}, $\phi_{\alpha_i}(x_i)$ equals to $1$ if $\alpha_i=x_i$, otherwise equals to $0$.
\end{definition}

\begin{theorem}[Walsh-Fourier Expansion~\cite{o2014analysis}]\label{thm:WFE}
    Having fixed the Fourier bases, $f$ is uniquely expressible as 
    \begin{equation*}
        f(X)=\sum_{\alpha\in{}D}\left(
            \hat{f}(\alpha)\cdot\phi_{\alpha}(X)
        \right)
    \end{equation*}
    where $\hat{f}(\alpha) \in \mathbb{R}$ is called Fourier coefficient, and it satisfies
    \begin{equation*}
        \hat{f}(\alpha)
        =\mathop{\mathbb{E}}_{X\sim D}\left[
            f(X)\cdot\phi_{\alpha}(X)
        \right].
    \end{equation*}
\end{theorem}

\section{Theoretical Framework}
In this section, we first recap the previous \ac{CLS} for Boolean SAT solving.
We then demonstrate how our proposed solver, \textsc{FourierCSP}, can extend the CLS framework to general constraint satisfaction problems. 
Proofs are delayed to Appendix~\ref{sec:proof}.

\subsection{Continuous Local Search}

\ac{CLS}-based solvers define a continuous objective function, the optima of which encode the solutions to the original problem defined on discrete variables. 
The \ac{CLS} approach for SAT solving was initially proposed in \textsc{UniSAT}~\cite{gu1994global}.
Modern variants, like \textsc{FourierSAT}~\cite{kyrillidis2020fouriersat}, \textsc{GradSAT}~\cite{kyrillidis2021continuous} transform Boolean constraints to continuous functions, called \ac{WFE} via the Walsh-Fourier transform~\cite{o2014analysis}.
\textsc{FastFourierSAT} is the parallel version of \textsc{FourierSAT} and achieve more than 100$\times$ speedup in gradient computation~\cite{cen2025massively}.
These approaches solve the following objective, and the algorithmic framework is described in Algorithm~\ref{alg:cls}.
\begin{definition}[Objective]\label{def:obj}
    Given a formula $F=\bigwedge_{c\in{}C}c$, every constraint $c\in{}C$ can be transformed to a function by Theorem~\ref{thm:WFE}.
    We denote the corresponding expansion as $f_c$ and the objective function w.r.t. $F$ by 
    \begin{equation*}
        \mathcal{F} = -\sum_{c\in{}C}f_c.
    \end{equation*}
\end{definition}

\begin{algorithm}[tb]
    \caption{Continuous Local Search}
    \label{alg:cls}
    \textbf{Input}: A formula $F$ with a constraint set $C$\\
    \textbf{Output}: A discrete assignment $X\in D$
    
    \begin{algorithmic}[1]
        \STATE Sample $P$ from $\tilde{D}$
        \FOR {$j = 1, \cdots, \texttt{max\_iter}$}
            \STATE Search for a local optimum $P^*$ of $\mathcal{F}(P)$
            \STATE Discretized assignment: $X=\mathcal{R}(P^*)$
            \IF {$\mathcal{F}(X)=-|C|$}
                \STATE \textbf{return} $X$
            \ENDIF 
        \ENDFOR
        \STATE \textbf{return} $X$ with the highest $\mathcal{F}(X)$
    \end{algorithmic}
\end{algorithm}

The entire lineage of \ac{CLS} solvers, from the 1990s approaches~\cite{gu1994global} to modern variants, is exclusively designed for Boolean constraints.
While they perform well on certain problem classes, these solvers often fail to scale effectively to large or complex \ac{CSP} instances.
This limitation becomes particularly acute for finite-domain problems, which require conversion to potentially larger Boolean formulas.
This limitation fundamentally restricts their practical application for real-world problems that naturally require finite-domain representations.

\subsection{FourierCSP}
To address the above limitation, we propose \textsc{FourierCSP}, a CLS framework that directly supports finite-domain CSPs without requiring Booleanization. 
We start by applying randomized rounding to the generalized form of \ac{WFE} in Theorem~\ref{thm:WFE}, which reduces the CSP to a differentiable multilinear function over a product of probability simplices.
In particular, although stationary points may occur in the interior of the domain, we prove that \emph{every local minimum of the multilinear objective must lie on the boundary} of the simplices, ensuring the solution quality.
Building on this differentiable formulation, we analyze the behavior of \ac{PGD} and \ac{MD} and derive practical convergence bounds for both methods.
Based on these theoretical insights, we propose a \ac{HD} that combines \ac{PGD} and \ac{MD}, exploiting their complementary geometries to accelerate convergence.

\paragraph{Reduction by Randomized Rounding}
Theorem~\ref{thm:WFE} illustrates how to expand discrete functions to multilinear polynomials through \ac{WFE}.
The resulting Fourier coefficients $\hat{f}(\alpha)$ provide valuable insights into analyzing the discrete function properties. 
However, the expanded function $f$ remains discrete and non-differentiable given the Fourier Bases in Definition~\ref{def:bases}, limiting its utility for gradient-based optimization. 
To enable continuous optimization, we introduce a randomized rounding $\mathcal{R}$ that relaxes discrete variable assignments into a continuous domain.
Consider a set of finite domain variables $X=\{x_1,...,x_n\}$.
The vector probability space of $x_i$ is a simplex denoted by 
\begin{equation*}
    \tilde{D}_i := \left\{p_{i,j} \Bigg| \min_{j\in\mathrm{dom}(x_i)}p_{i,j}\ge{}0,\sum_{j\in\mathrm{dom}(x_i)}p_{i,j}=1\right\}.
\end{equation*}
The collection of all simplexes forms a continuous search space $\tilde{D}=\left(\tilde{D}_1\times\cdots\times\tilde{D}_n\right)$.

\begin{definition}[Randomized Rounding]\label{def:rounding}
    Given a finite domain $D$.
    The randomized rounding function, denoted by $\mathcal{R}:\tilde{D}\to{}D$ is defined by
    \begin{equation*}
        \mathbb{P}[\mathcal{R}(P)_i=j]=p_{i,j}
    \end{equation*}
    for all $j\in\mathrm{dom}(x_i)$, $p_i\in\tilde{D}_i$ and $i\in[n]$.
\end{definition}

Intuitively, if $p_{i,j}$ is closer to $1$, then $\mathcal{R}(a)_i$ is more likely to be rounded to $j$.
Using this randomized interpretation, we can define a probability distribution over full assignments by rounding $\mathcal{R}$ of each variable independently.

\begin{definition}[Probability Space]\label{def:prob_space}
    We define a probability space on a discrete matrix w.r.t. to real point $P\in\tilde{D}$, denoted by $\mathcal{S}:D\to[0,1]$, as:
    \begin{equation*}
        \mathcal{S}_P(X)=\mathbb{P}[\mathcal{R}(P)=X]=\prod_{i=1}^n p_{i,x_i},
    \end{equation*}
    for all $X\in{}D$, with respect to the randomized-rounding function $\mathcal{R}$.
    We use $X\sim\mathcal{S}_P$ to denote that $X$ is sampled from the probability space $\mathcal{S}_P$.
\end{definition}

The probability space allows us to define a continuous extension of any discrete function by taking its expectation over a randomized assignment.

\begin{theorem}[Expectation]\label{thm:expectation}
    Given a Walsh-Fourier expansion $\mathcal{F}$, for a real point $P\in\tilde{D}$, we have
    \begin{equation}\label{eq:expectation}
    \begin{aligned}
        \mathop{\mathbb{E}}_{X\sim\mathcal{S}_P}[\mathcal{F}(X)]
            & = \mathop{\mathbb{E}}_{X\sim\mathcal{S}_P}[-\sum_{c\in{}C}f_c(X)] \\
            & = -\sum_{c\in{}C}\sum_{\alpha}\left(
            \hat{f}_c(\alpha)\cdot\prod_{i=1}^{n}p_{i, \alpha_i}
            \right)\\
            & = \mathcal{F}(P).
    \end{aligned}
    \end{equation}
\end{theorem}
We can observe from Theorem~\ref{thm:expectation} that $\mathcal{F}(P)$ is a multilinear polynomial.
Therefore, we relax the domain from discrete $D$ to continuous $\tilde{D}$, and it leads to a $\rho$-Lipschitz, and $L$-smooth multilinear continuous function.

\begin{proposition}[Lipschitz]\label{prop:lipschitz}
    Let $\mathcal{F}: \tilde{D}\to[-|C|,0]$ be a multilinear polynomial.
    Let $n$ be the number of variable, then $N=\sum_{i=1}^n|\mathrm{dom}(x_i)|$.
    For every point $P, P'\in\tilde{D}$, we have
    \begin{align*}
        |\mathcal{F}(P)-\mathcal{F}(P')| &\le \rho \|P-P'\|,\\
        |\nabla\mathcal{F}(P)-\nabla\mathcal{F}(P')| &\le L\|P-P'\|,
    \end{align*}
    where $\rho=|C|N^\frac{1}{2}$, $L=|C|N$.
\end{proposition}

As the expectation is a continuous and smooth multilinear polynomial, we formalize the reduction in Theorem~\ref{thm:reduction}.

\begin{theorem}[Reduction]\label{thm:reduction}
    Given a constraint satisfaction problem of formula $F$, which is the conjunction of all constraints $c\in{}C$.
    $F$ is satisfiable if and only if 
    \begin{equation}\label{eq:reduction}
        \min_{P\in\tilde{D}} \left(\mathop{\mathbb{E}}_{X\sim\mathcal{S}_P}[\mathcal{F}(X)]\right) = -|C|.
    \end{equation}
\end{theorem}

\subsection{Optimization Approaches}
The essence of \ac{CLS} is using continuous optimization to find the minima of a non-convex but smooth function, where the smoothness of the \ac{WFE} was analyzed in Proposition~\ref{prop:lipschitz}.
Given the inherent difficulty of optimizing non-convex functions~\cite{jain2017non}, it is often sufficient in practice to converge to a local optimum and subsequently verify, via Theorem~\ref{thm:reduction}, whether it corresponds to a global solution.
During continuous optimization, the descent steps move the assignment toward regions of lower objective value, and the iterate eventually reaches a local minimum at which the suboptimality gap $\mathcal{F}(P_t)-\mathcal{F}(P^*)$ becomes small, typically $\le \epsilon$.
A potential concern is the existence of local minima strictly inside the feasible region, \emph{i.e.}, in $\tilde{D} \setminus \delta\tilde{D}$, since such points would provide no guarantee about the quality of the decoded discrete assignment.
However, Lemma~\ref{lmm:saddle} establishes that no interior local maxima exist for our multilinear objective, and thus all optimal points lie on the boundary.
This structural property directly leads to Theorem~\ref{thm:optimality}.

\begin{lemma}\label{lmm:saddle}
    Let $\tilde{D}_\perp$ be a hyperplane with $\vec{1}$ as the normal vector.
    For every non-constant $\mathcal{F}$ and a point $P\in\tilde{D}\backslash{\delta\tilde{D}}$, there exist a region size $\epsilon > 0$ and a directions, $V\in\tilde{D}_\perp$ such that for all $\eta\in(0,\epsilon)$, the following holds.
    \begin{equation*}
        \mathcal{F}(P-\eta{}V) < \mathcal{F}(P)
    \end{equation*}
\end{lemma}

\begin{theorem}[Optimality]\label{thm:optimality}
    Any interior point $P\in\tilde{D}\backslash{}\delta\tilde{D}$ cannot be a local optimum of $\mathcal{F}(P)$.
    The optimal points are only attained from the boundary $\delta\tilde{D}$.
\end{theorem}

Having established these optimization characteristics, the next question is how different descent geometries behave within this landscape.
In particular, \ac{PGD} and \ac{MD} offer distinct update rules and induce different optimization geometries on the probability simplices.

\paragraph{Projected Gradient Descent}
Since Eq.~\eqref{eq:expectation} is defined only over the constrained domain $\tilde{D}$, a single gradient step may result in an intermediate point that falls outside this domain. 
Therefore, each gradient step, \emph{i.e.}, Eq.~\eqref{eq:PGD-grad}, must be followed by a projection, \emph{i.e.}, Eq.~\eqref{eq:PGD-proj}, back onto $\tilde{D}$ to ensure feasibility.
\begin{align}
    \label{eq:PGD-grad} Q_{t+1} & = P_t - \eta\cdot\nabla\mathcal{F}(P_t) \\
    \label{eq:PGD-proj} P_{t+1}  & = \mathrm{proj}_{\tilde{D}}(Q_{t+1})
\end{align}

Based on the randomized rounding in Definition~\ref{def:rounding}, the Euclidean projection of $Q\in\mathbb{R}^{|\mathrm{dom}(x_1)|\times\cdots\times|\mathrm{dom}(x_n)|}$, onto $P\in\tilde{D}$, is defined as
\begin{equation}\label{eq:proj}
    \mathrm{proj}_{\tilde{D}}(Q) = \argmin_{P}\left(I_{\tilde{D}}(P) + \frac{1}{2}||Q-P||_2^2\right),
\end{equation}
where $I_{\tilde{D}}$ is an indication function. 
$I_{\tilde{D}}(P)$ is $\infty$ (resp. $0$) when $P\notin\tilde{D}$ (resp. $\in$).

The projection operation $\mathrm{proj}_{\tilde{D}}(Q)$ can be computed efficiently by separately projecting each $q_i$ onto its respective simplex $\tilde{D}_i$, leveraging the product structure of $\tilde{D}$.
Each sub-projection corresponds to a standard simplex projection problem~\cite{chen2011projection}.

\begin{theorem}\label{thm:pgd}
    With the step size $\eta=\frac{1}{L}$, the \ac{PGD} will converge to a $\epsilon$-critical point in $\mathcal{O}\left(\frac{|C|\cdot L}{\epsilon^{2}}\right)$ iterations.
\end{theorem}


\paragraph{Mirror Descent}
While projected gradient methods rely on Euclidean geometry for projection, \ac{MD} generalizes this process to non-Euclidean geometries induced by a strictly convex potential function $\psi$.
\begin{align*}
    Q_{t+1} &= \nabla\psi(P_{t}) - \eta\nabla\mathcal{F}(P_{t}),\\
    P_{t+1} &= \nabla\psi^{-1}(Q_{t+1}).
\end{align*}

We choose negative entropy as the potential function for each probability simplex $\tilde{D}_i$, hence $\psi(P) = {\langle P, \log(P) \rangle}$.
Then we have
\begin{align*}
    \nabla\psi(P) =\log(P), \quad \nabla\psi^{-1}(Q)= \frac{\exp(Q)}{\langle \exp(Q), \mathbf{1} \rangle}.
\end{align*}

\begin{theorem}\label{thm:md}
    Let $\Theta=\sum_{i=1}^n\log|\mathrm{dom}(x_i)|$. 
    Let $T$ be the number of \ac{MD} steps.
    with the step size $\eta\le\frac{\sqrt{2\Theta}}{\rho\sqrt{T}}$, then \ac{MD} will converge to a $\epsilon$-critical point in $\mathcal{O}\left(\frac{2\Theta\cdot\rho}{\epsilon^2}\right)$ iterations.
\end{theorem}

\paragraph{Hybrid Descent}
Theorem~\ref{thm:pgd} implies a larger gradient norm $\|\nabla\mathcal{F}\|$ leads to a greater decrease in the objective at each projected gradient step, and therefore accelerates convergence.
In contrast, Theorem~\ref{thm:md} requires a global bound $\|\nabla\mathcal{F}\|\le\rho$, and a larger gradient increases this bound, which forces a smaller step-size in the analysis and ultimately slows the convergence rate of \ac{MD}.
This contrast is made explicit in~\cite{allen2014linear}, which shows that \ac{PGD} and \ac{MD} operate under different geometries and therefore complement each other.
Motivated by this observation, we adopt a hybrid update rule, \ac{HD}, that evaluates both descent directions and selects the better one:
\begin{align*}
    &P_{PGD} = \mathrm{proj}_{\tilde{D}}(P_t - \eta_{PGD}\cdot\nabla\mathcal{F}(P_t)), \\
    &P_{MD}= \nabla\psi^{-1}(\nabla\psi(P_{t}) - \eta_{MD}\nabla\mathcal{F}(P_{t})), \\
    &P_{t+1}=\argmin_{P \in \{\,P_{\mathrm{PGD}},\,P_{\mathrm{MD}}\}} \mathcal{F}(P).
\end{align*}

\paragraph{Implementation Details}
\textsc{FourierCSP} is implemented by extending the core \textsc{jaxopt.ProjectedGradient} framework and incorporating the projection and mapping functions from \textsc{jaxopt.MirrorDescent}~\cite{jax2018github,blondel2022efficient}. 
We adopt the default settings in \textsc{jaxopt}, using a maximum of 500 gradient steps and a tolerance of 0.001.
The step size is determined automatically via backtracking line search, with a maximum of 15 backtracking steps and a decrease factor of 0.5.
All descent methods are evaluated both with and without \ac{FISTA} acceleration.

\section{Experimental Section}
In this section, we compare our tool, \textsc{FourierCSP}, with versatile state-of-the-art solvers. 
The objective of our experiments is to answer the following research questions.

\noindent\textbf{RQ1.} 
Given their Boolean-encoded formulas of CSP, how well does \textsc{FourierCSP} perform in solving them?

\noindent\textbf{RQ2.} 
How do different descent methods compare in effectiveness and robustness in the \textsc{FourierCSP} framework?

\noindent\textbf{RQ3.} 
How does \textsc{FourierCSP} compare to SAT, ILP, and CP solvers in solving constraint satisfaction problems that the constraints have structural patterns?

\noindent\textbf{RQ4.} 
How does \textsc{FourierCSP} perform relative to SAT, ILP, and CP solvers in solving constraint optimization problems that the constraints have structural patterns?

\begin{figure*}[t]
    \centering
    \subfigure[\#Iter to local minimum]{
        \includegraphics[width=0.23\linewidth]{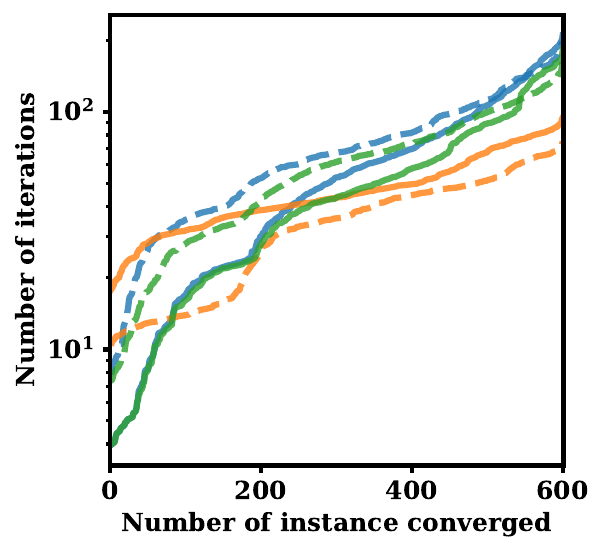}
    }\hfill
    \subfigure[Time to local minimum]{
        \includegraphics[width=0.23\linewidth]{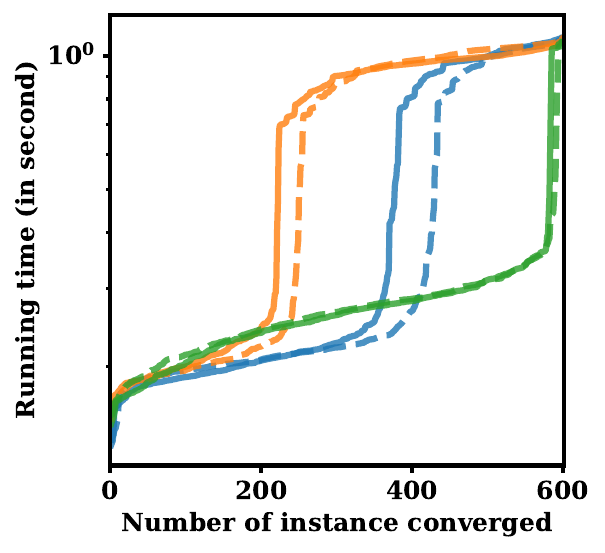}
    }\hfill
    \subfigure[\#Iter to global minimum]{
        \includegraphics[width=0.23\linewidth]{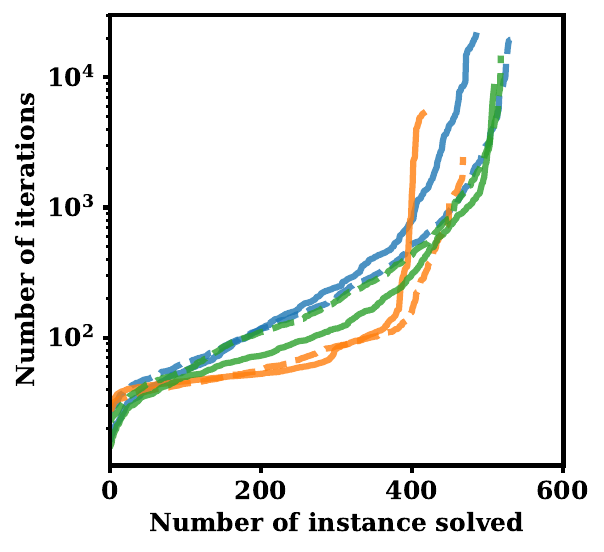}
    }\hfill
    \subfigure[Time to global minimum]{
        \includegraphics[width=0.23\linewidth]{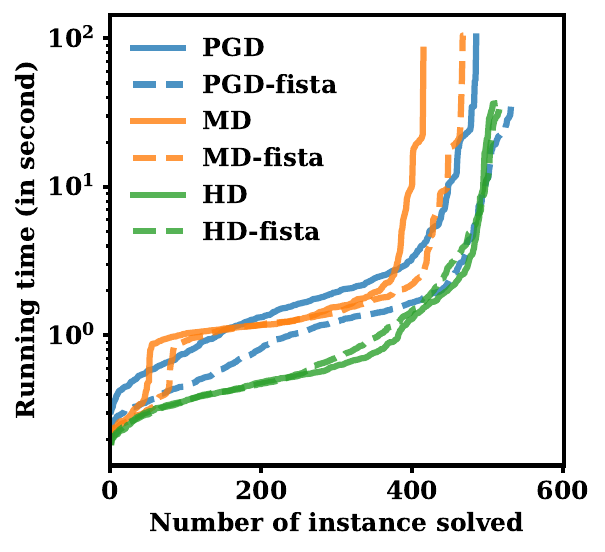}
    }
    \caption{The performance of \textsc{FourierCSP} with different descent methods on the random benchmark.}
    \label{fig:random}
\end{figure*}

\paragraph{Benchmarks and Encodings}
To fully exploit the potential of Fourier expansion-based CLS, benchmark instances should ideally be symmetric, \emph{i.e.}, their constraints should exhibit similar structural patterns.
This symmetry enables efficient GPU acceleration by allowing the solver to fully utilize parallelism, achieving up to 100+ times speedup in gradient computation compared to CPU~\cite{cen2025massively}.
The benchmarks include constraint satisfaction and optimization problems.
The first benchmark consists of \textbf{random hybrid constraints}, combining \emph{greater-than}, \emph{not-all-equal}, and \emph{modulo} constraints.
For each combination of $n_\text{var}\in\{100,...,1000\}$ and $n_\text{state}\in\{4,8,16,32\}$, we generate 10 random instances. 
Each instance contains $\frac{n_\text{var}}{10}$ of each type.
The arity (length) of NAE and modulo constraints is fixed to 10.
The second benchmark is \textbf{task scheduling problems}, which can be useful for the trace scheduler of multicore hardware systems.
For each combination of $N_{\text{worker}} \in \{32, 64, 128, 256, 512\}$ and $m \in \{4, 8, 16\}$, we set the number of jobs to $N_{\text{job}} = m \cdot N_{\text{worker}}$, and generate 10 random instances accordingly.
The third benchmark, \textbf{graph coloring with hashing queries}, is an optimization problem motivated by approximate model counting~\cite{pote2025towards}, where hash functions are employed to partition the solution space. 
We use graph coloring as the underlying problem and add parity (modulo 2) constraints to the formula to serve as hash functions.
For each $N_\text{node} \in \{512, 1024, 2048, 4096\}$ and $N_\text{color} \in \{8, 16, 32, 64\}$, we generate 10 random instances. 
Additionally, $N_\text{node}$ modulo 2 constraints are added to each instance. 

Benchmark instances are linearized for ILP solvers and Booleanized for Boolean SAT solvers. 
Further details on the benchmarks and the encodings used for different solvers can be found in the appendix, where we show that the Booleanization will naturally lead to pseudo-Boolean and XOR constraints.
Further encoding the largest Booleanized instance into CNF~\cite{ignatiev2018pysat,li2000integrating} yields a formula comprising over 160 million variables and 710 million clauses. 
Although modern SAT solvers are highly optimized and capable of solving large benchmarks, such encodings are extremely memory-inefficient; for instance, the corresponding CNF text file is around 16 GB in size.
The details of the encoding are in Appendix~\ref{sec:benchmark_encode}.

\paragraph{Solver Competitors}
\textbf{(1)} \textsc{FourierSAT}~\cite{kyrillidis2020fouriersat}: A CLS solver designed for Boolean formulas that are conjunctions of symmetric Boolean constraints.
In~\citep{cen2025massively}, it uses \ac{PGD} with \ac{FISTA} as the optimizer. 
\textbf{(2)} \textsc{FourierMaxSAT}~\cite{cen2025massively}: An extension of FourierSAT incorporating an SAT solver for warm-start initialization.
\textbf{(3)} \textsc{LinPB}~\cite{yang2021engineering}: A SAT solver for linear pseudo-Boolean and XOR constraints. 
It employs RoundingSAT, which extends CDCL to pseudo-Boolean reasoning, as the underlying pseudo-Boolean solver. XOR constraints are supported via Gauss-Jordan elimination, adapted from CryptoMiniSAT.
\textbf{(4)} \textsc{Gurobi-ILP}~\cite{gurobi}: An ILP solver that uses branch-and-bound, cutting planes, and heuristics to handle integer linear constraints efficiently. 
The solver also leverages preprocessing and parallel computation to enhance performance. 
We use version 11.0.3.
\textbf{(5)} \textsc{OR-Tools CP-SAT}~\cite{perron_et_al:LIPIcs.CP.2023.3}: A hybrid solver that integrates a clause-learning SAT core, constraint programming, and linear relaxation within a unified model to solve combinatorial problems using advanced techniques; it won gold medals in all categories of the MiniZinc Challenge 2024.
This solver won gold medals in all categories in the MiniZinc Challenge 2024.
We use version 9.11.
\textbf{(6)} \textsc{ParaILP}~\cite{ijcai2024p768}: A parallel discrete local search solver for ILP.

\paragraph{Experimental Setups}
Solvers run on a high-performance cluster node with dual AMD EPYC 9654 CPUs, 24$\times$96 GB DRAM, and 8 NVIDIA L40S GPUs.
We allow \textsc{Gurobi-ILP}, \textsc{OR-Tools CP-SAT}, and \textsc{ParaILP} to use maximally 256 threads. 
As SAT solvers are particularly memory-intensive due to clause database management~\cite{alekhnovich2000space,elffers2016trade}, simple parallelization via multi-threaded racing, without inter-thread clause sharing, typically yields limited performance improvements. 
Consequently, \textsc{LinPB}, which does not support multithreading, is executed in single-threaded mode. 
The CLS solvers, \textsc{FourierSAT} and \textsc{FourierCSP}, are executed on the GPU to exploit their inherent parallelism and leverage the computational advantages of modern GPU architectures.
We set 1000 seconds as the timeout for each instance.

We assess solver performance on satisfaction problems using the PAR-2 (penalized average runtime 2) score, a metric balancing runtime efficiency and robustness by imposing penalties on timeouts. 
\begin{equation*}
    \text{PAR-2} = \frac{1}{N}\sum_{i=1}^N\left(t_i\cdot{}\mathbf{1}[t_i \leq T]+2T\cdot\mathbf{1}[t_i > T]\right)
\end{equation*}
where $\mathbf{1}(\cdot)$ is the indicator function and $T$ is the time limit.

We evaluate the optimization problem using two metrics: the relative score and the number of instances for which a solver attains the best objective value (\#Win).
The relative score is defined as follows.
\begin{equation*}
    \text{score}(s, i)=\frac{1+c_i^s}{1+c_i^*}
\end{equation*}
where $c_i^s$ denotes the result obtained for instance $i$ by algorithm $s$, and $c_i^*$ denotes the best-known result.

\begin{figure*}[t]
    \centering
    \subfigure[\textsc{FourierCSP} with different descent methods.]{
        \includegraphics[width=0.46\linewidth]{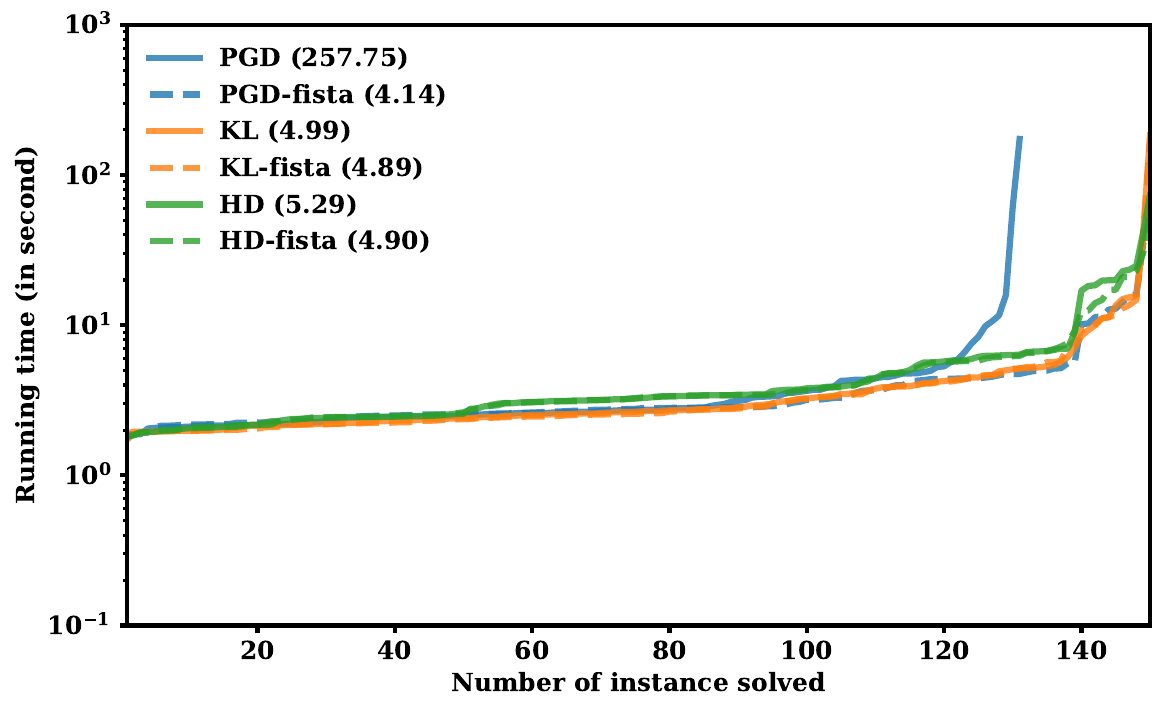}
    }\hfill
    \subfigure[Comparison of \textsc{FourierCSP} with other solvers.]{
        \includegraphics[width=0.46\linewidth]{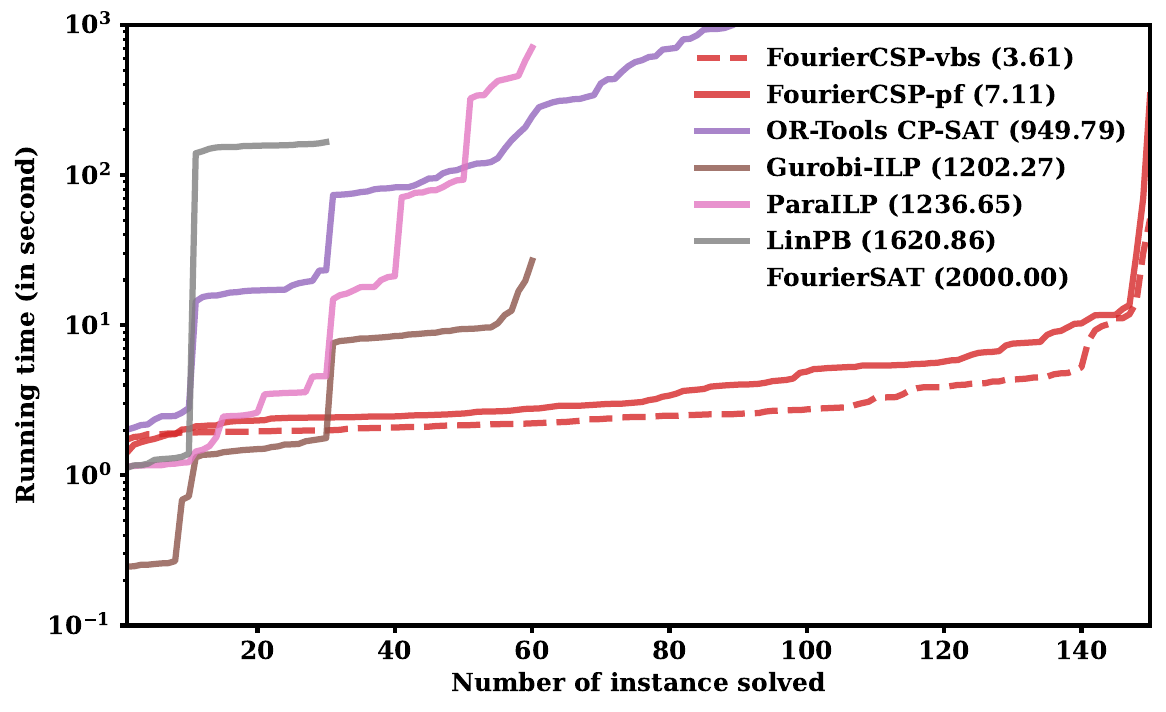}
    }
    \caption{Results on task scheduling problems. The numbers in the legend are the average PAR-2 score of the solvers.}
    \label{fig:scheduling}
\end{figure*}

\subsection{Answer to RQ1. }
Since \textsc{FourierCSP} is the multi-valued extension of \textsc{FourierSAT}, the semantics of solving hybrid SAT formulas remain equivalent to those of \textsc{FourierSAT}.
Prior research indicates that \textsc{FourierSAT} demonstrates strong performance on benchmarks natively formulated as hybrid SAT problems. 
However, our experiments reveal a significant limitation: \textsc{FourierSAT} fails to solve any instances encoded from both scheduling (Fig.~\ref{fig:scheduling}) and graph coloring benchmarks (Fig.~\ref{fig:color}). 
This observation underscores a critical gap in the capability of \textsc{FourierSAT} when handling problems not inherently structured as conjunctions of symmetric Boolean constraints.
Consequently, it highlights the necessity of extending the solver to a more general domain.
Specifically, the results motivate the development of \textsc{FourierCSP}, a multi-valued extension of \textsc{FourierSAT}, which is designed to address these limitations precisely by supporting more general, \emph{i.e.}. multi-valued formulations.

\subsection{Answer to RQ2.}
From Fig.~\ref{fig:random} (a), \ac{MD} requires the fewest gradient iterations to reach a local minimum.
However, Fig.~\ref{fig:random} (b) shows that it nevertheless takes longer in wall-clock time.
This is because all methods use line search to select the step size: although \ac{MD} makes rapid progress per gradient step, its updates satisfy the line-search conditions less easily, causing more backtracking.
As a result, each \ac{MD} iteration is more expensive, leading to a higher total runtime despite fewer iterations.
It's counterintuitive that \ac{HD} is the most time-efficient approach, where each gradient step needs to measure the gradient step made by both \ac{PGD} and \ac{MD}. 
This is because \ac{HD} inherits the faster descent direction of the two while avoiding the costly backtracking behavior that slows down pure \ac{MD}.
The same pattern appears in Fig.~\ref{fig:random}(c–d) for convergence to the global minimum.

Across all settings, \ac{HD} is the most effective and robust method within the \textsc{FourierCSP} framework for constraint satisfaction problems, consistently solving the largest number of instances under a fixed trial budget.
\ac{FISTA} further accelerates both \ac{PGD} and \ac{MD}, with PGD-fista performing competitively with \ac{HD} and \ac{HD}-fista.
This trend is confirmed on the scheduling benchmark in Fig.~\ref{fig:scheduling} (a), where plain \ac{PGD} is significantly slower than all other descent methods, while \ac{PGD}-fista eliminates most of this gap, and even achieves the lowest PAR-2 score.

The results for the constraint optimization benchmark are reported in Table~\ref{tab:color}.
Although \ac{PGD}-fista achieves the highest anytime score, \ac{MD}-fista yields the largest number of best solutions.
Moreover, other descent methods also achieve the best result on certain instances, indicating that no single method uniformly dominates across all problems.
These observations suggest that a portfolio strategy combining multiple descent methods is more reliable and robust than relying on any individual optimizer.

In the remainder of the paper, \textsc{FourierCSP}-vbs denotes the virtual-best solver across all descent variants, \emph{i.e.}, a setting that evaluates each descent method in parallel using multiple GPUs.
In contrast, \textsc{FourierCSP}-pf refers to a sequential portfolio strategy that runs all descent methods on a single GPU.
As illustrated in Fig.~\ref{fig:scheduling}(b), although \textsc{FourierCSP}-vbs employs 6 GPUs, its speedup over the single-GPU portfolio \textsc{FourierCSP}-pf remains below 2$\times$. 
This phenomenon reflects a natural boundary effect: the fine-grained parallelism available within each constraint’s WFE evaluation is already saturated by a single GPU, leaving limited room for further acceleration. 
Additional GPUs mainly enable evaluating multiple random initializations with different descent methods concurrently, which improves robustness but cannot provide unbounded speedup.

\begin{table}[t]
    \centering\begin{tabular}{ccccc}\toprule
    {Descent Method}      && {Anytime Score}      && {\#Win}    \\ 
    \midrule
    {PGD}         && 0.87                 &&  26             \\
    {PGD}-fista   && \textbf{0.97}        &&  15             \\
    {MD}          && 0.80                 &&  20             \\
    {MD}-fista    && 0.76                 &&  \textbf{46}    \\
    {HD}          && 0.86                 &&  11             \\
    {HD}-fista    && 0.92                 &&  13             \\
    \bottomrule\end{tabular}
    \caption{Performance of \textsc{FourierCSP} with different descent methods on the graph coloring with hashing query benchmark.}
    \label{tab:color}
\end{table}

\begin{figure*}[t]
    \centering
    \includegraphics[width=1.0\linewidth]{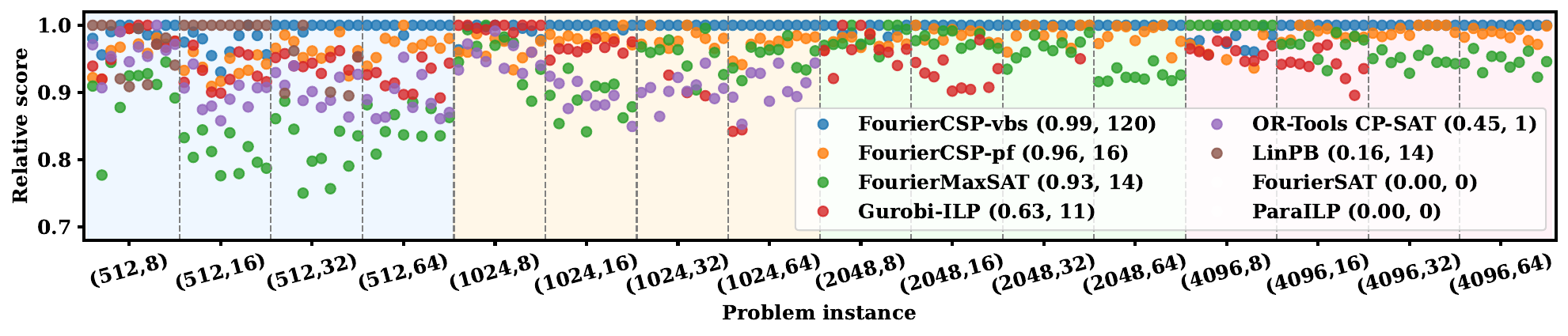}
    \caption{Results on graph coloring with hashing query problems.
    The numbers in the legend are (the average relative score, \#Win) of the solvers.}
    \label{fig:color}
\end{figure*}

\subsection{Answer to RQ3.}
Fig.~\ref{fig:scheduling} presents the results on the second benchmark, constraint satisfaction problems. 
\textsc{FourierCSP} and \textsc{ParaILP} represent two parallel local search solvers based on continuous and discrete optimization methods, respectively. 
Under a timeout limit of 1000 seconds, \textsc{FourierCSP} successfully solves all 150 instances, substantially surpassing \textsc{ParaILP}, which solves only 60.
This significant performance gap can be attributed to the differences in the local search methods.
Discrete local search methods require that each assignment strictly increase the number of satisfied constraints, often leading to plateaus, particularly for constraints that are long or structurally complex, where small perturbations to discrete variable assignments may not immediately yield progress. 
In contrast, \ac{CLS} methods can exploit fine-grained improvements: the optimizer advances as long as the continuous objective decreases, even marginally.

Among complete solvers, \textsc{OR-Tools CP-SAT}, \textsc{Gurobi-ILP}, and LinPB solve 89, 60, and 30 instances, respectively. 
Although \textsc{OR-Tools CP-SAT} supports channeling constraints, it requires auxiliary variables for modeling conditional logic (\emph{e.g.}, if-then-else), which can substantially expand the search space. 
A key advantage of \textsc{FourierCSP} is that it evaluates the Walsh-Fourier expansion directly, without introducing auxiliary variables, thus maintaining a lower-dimensional search space and preserving problem structure.
As a result, \textsc{FourierCSP}-pf achieves a 133.59$\times$ speedup in PAR-2 score compared to \textsc{OR-Tools CP-SAT}.

Instances solved by the ILP solvers are generated with $N_{\text{worker}} \leq 256$, and the resulting ILP formulas contain fewer than 130,000 variables and 162,000 linear constraints. 
Although \textsc{ParaILP} is an incomplete ILP solver, typically known for its efficiency in finding feasible solutions, \emph{i.e.}, it underperforms \textsc{Gurobi-ILP}, as reflected by its higher PAR-2 score.
Instances solved by LinPB are generated with $N_{\text{worker}} \leq 128$, where the encoded PB-XOR formulas contain fewer than 65,000 Boolean variables and 75,000 constraints. 
Despite its reasoning capabilities, LinPB must handle much larger formulas in practice: problems with $N_{\text{worker}} \leq 128$ result in formulas with approximately 295,000 variables and 325,000 constraints. 
\textsc{FourierCSP}-pf achieves a 169.10$\times$ and 227.97$\times$ speedup compared to \textsc{Gurobi-ILP} and \textsc{LinPB}, respectively. 
These results suggest that when solving \ac{CSP} involving non-Boolean variables and non-linear constraints, a solver that can process more compact formulas, \emph{i.e.}, without bloating the search space, can offer significant performance advantages.

\subsection{Answer to RQ4.}
This benchmark includes a substantial number of modulo constraints, which are known to be challenging for \textsc{OR-Tools CP-SAT}~\cite{krupke2024cpsat}. 
While modulo constraints remain beyond the expressive or algorithmic reach of most conventional solvers, \textsc{FourierCSP} can natively accommodate them through its uniform Fourier expansion framework.
As shown in Fig.~\ref{fig:color}, \textsc{FourierCSP}-pf outperforms all competing solvers on this benchmark, achieving a relative score of 0.96. 
The \textsc{FourierCSP}-vbs setting reflects the performance upper bound of \textsc{FourierCSP}-pf and attains the best solution on 120 out of 160 instances.
In contrast to its competitive performance on the previous benchmark, \textsc{OR-Tools CP-SAT}’s relative score is 0.18 lower than that of \textsc{Gurobi-ILP} for these instances.
This relative underperformance can likely be attributed to \textsc{OR-Tools CP-SAT}'s inefficiency in processing modulo constraints.

Under binary logarithm encoding, a modulo 2 constraint can be easily Booleanized to an XOR constraint. 
Although \textsc{LinPB} can natively solve XOR constraints using Gaussian elimination, it achieves a relative score of 0.16 on this benchmark. 
Nevertheless, \textsc{LinPB} contributes more to the virtual best solver than \textsc{OR-Tools CP-SAT}, as it attains the best solution for 14 instances.
Notably, formulas with parity constraints are known to be challenging for DLS solvers~\cite{crawford1994minimal}. 
Unsurprisingly, \textsc{ParaILP} fails to provide valid solutions for any instance in this benchmark.

Previous work demonstrated that \textsc{FourierSAT} performs competitively on graph problems when the encoded formulas comprise exclusively symmetric Boolean constraints~\cite{cen2025massively}. 
In contrast, the Booleanized formulas used in this benchmark consist primarily of pseudo-Boolean constraints, which would require further symmetrization into disjunctive clauses to fit the \textsc{FourierSAT} framework. 
This additional encoding step increases the complexity, making the problem significantly more challenging for \textsc{FourierSAT}, which does not provide valid solutions for any instance in this benchmark.
While \textsc{FourierSAT} encounters difficulties when addressing formulas containing a large number of clauses, \textsc{FourierMaxSAT} adopts a hybrid approach. 
It leverages an SAT solver to solve part of the formula, \emph{i.e.}, the disjunctive clauses, and subsequently uses CLS to improve the solution, \emph{i.e.}, by attempting to satisfy more XOR constraints.
Eventually, \textsc{FourierMaxSAT} achieved a relative score of 0.93.

Moreover, we observe that the instances for which \textsc{FourierCSP} attains the best solutions are typically those with larger problem sizes. This observation suggests a qualitative distinction in scalability: while systematic search solvers tend to be effective for smaller instances, \textsc{FourierCSP} exhibits improved scalability and maintains solution quality as instance size increases.


\section{Conclusion and Future Directions}
In this paper, we introduced \textsc{FourierCSP}, a novel continuous optimization framework for constraint satisfaction problems.
By extending the Walsh-Fourier expansion to finite-domain variables, our approach can handle versatile constraints without the need for memory-intensive encodings or auxiliary variables. 
By randomized rounding, the discrete constraints are transformed into differentiable functions, enabling gradient-based optimization.
We adopted multiple gradient-based approaches for our optimization framework, which are supported by theoretical guarantees on convergence.
Empirical evaluations on benchmark suites demonstrated that \textsc{FourierCSP} is scalable and performs competitively. 
The experimental results indicate that \textsc{FourierCSP} is a substantial extension of \textsc{FourierSAT}, broadening the class of problems that can be efficiently solved using \ac{CLS} techniques. 
Ultimately, the differentiable nature of \textsc{FourierCSP} facilitates seamless integration into neurosymbolic architectures, bridging the long-standing gap between data-driven deep learning and constraint reasoning.

The performance of Fourier transform-based \ac{CLS} relies heavily on GPU parallelism. 
An important direction for future work is to investigate how this parallelism can be effectively leveraged when constraints exhibit non-structural patterns. 
One promising approach is to leverage compact symbolic data structures, such as decision diagrams, to represent constraints and compute their circuit-output probabilities efficiently (Appendix~\ref{sec:dd}).
This would further generalize \textsc{FourierCSP}, enabling its application to a broader range of \ac{CSP} problems.



\section*{Acknowledgments}
We thank Moshe Vardi for the helpful and insightful comments and discussions that contributed to this work. 
This work was supported in part by the National Research Foundation under the Prime Minister’s Office, Singapore, through the Competitive Research Programme (project ID NRF-CRP24-2020-0002 and NRF-CRP24-2020-0003); in part by the Ministry of Education, Singapore, through the Academic Research Fund Tier 2 (project ID MOE-T2EP50221-0008); and in part by the National University of Singapore, through the Microelectronics Seed Fund (FY2024).
Zhiwei Zhang is supported in part by NSF grants (IIS-1527668, CCF1704883, IIS1830549), DoD MURI grant (N00014-20-1-2787), Andrew Ladd Graduate Fellowship of Rice Ken Kennedy Institute, and an award from the Maryland Procurement Office.

\bibliography{example_paper}

@inproceedings{kolaitis1998conjunctive,
  title={Conjunctive-query containment and constraint satisfaction},
  author={Kolaitis, Phokion G and Vardi, Moshe Y},
  booktitle={Proceedings of the seventeenth ACM SIGACT-SIGMOD-SIGART symposium on Principles of database systems},
  pages={205--213},
  year={1998}
}

@inproceedings{baluta2019quantitative,
  title={Quantitative verification of neural networks and its security applications},
  author={Baluta, Teodora and Shen, Shiqi and Shinde, Shweta and Meel, Kuldeep S and Saxena, Prateek},
  booktitle={Proceedings of the 2019 ACM SIGSAC Conference on Computer and Communications Security},
  pages={1249--1264},
  year={2019}
}

@article{brailsford1999constraint,
  title={Constraint satisfaction problems: Algorithms and applications},
  author={Brailsford, Sally C and Potts, Chris N and Smith, Barbara M},
  journal={European journal of operational research},
  volume={119},
  number={3},
  pages={557--581},
  year={1999},
  publisher={Elsevier}
}

@InProceedings{perron_et_al:LIPIcs.CP.2023.3,
  author =  {Perron, Laurent and Didier, Fr\'{e}d\'{e}ric and Gay, Steven},
  title =   {The CP-SAT-LP Solver},
  booktitle =   {29th International Conference on Principles and Practice of Constraint Programming (CP 2023)},
  pages =   {3:1--3:2},
  series =  {Leibniz International Proceedings in Informatics (LIPIcs)},
  ISBN =    {978-3-95977-300-3},
  ISSN =    {1868-8969},
  year =    {2023},
  volume =  {280},
  editor =  {Yap, Roland H. C.},
  publisher =   {Schloss Dagstuhl -- Leibniz-Zentrum f{\"u}r Informatik},
  address = {Dagstuhl, Germany},
  URL =     {https://drops.dagstuhl.de/opus/volltexte/2023/19040},
  URN =     {urn:nbn:de:0030-drops-190405},
  doi =     {10.4230/LIPIcs.CP.2023.3}
}

@misc{gurobi,
  author = {{Gurobi Optimization, LLC}},
  title = {{Gurobi Optimizer Reference Manual}},
  year = 2024,
  url = "https://www.gurobi.com"
}

@inproceedings{alekhnovich2000space,
  title={Space complexity in propositional calculus},
  author={Alekhnovich, Michael and Ben-Sasson, Eli and Wigderson, Avi},
  booktitle={Proceedings of the thirty-second annual ACM symposium on Theory of computing},
  pages={358--367},
  year={2000}
}

@inproceedings{elffers2016trade,
  title={Trade-offs between time and memory in a tighter model of CDCL SAT solvers},
  author={Elffers, Jan and Johannsen, Jan and Lauria, Massimo and Magnard, Thomas and Nordstr{\"o}m, Jakob and Vinyals, Marc},
  booktitle={Theory and Applications of Satisfiability Testing--SAT 2016: 19th International Conference, Bordeaux, France, July 5-8, 2016, Proceedings 19},
  pages={160--176},
  year={2016},
  organization={Springer}
}

@misc{krupke2024cpsat,
  title        = {The CP-SAT Primer: Using and Understanding Google OR-Tools' CP-SAT Solver},
  author       = {Dominik Krupke and Leon Lan and Michael Perk and others},
  year         = {2024},
  institution  = {TU Braunschweig},
  license      = {CC-BY-4.0},
  url          = {https://github.com/d-krupke/cpsat-primer},
  orcid        = {https://orcid.org/0000-0003-1573-3496}
}

@inproceedings{ijcai2024p768,
  title={ParaILP: A Parallel Local Search Framework for Integer Linear Programming with Cooperative Evolution Mechanism},
  author={Lin, Peng and Zou, Mengchuan and Chen, Zhihan and Cai, Shaowei},
  booktitle={Proceedings of the Thirty-Third International Joint Conference on Artificial Intelligence},
  pages={6949--6957},
  year={2024}
}

@article{schulte2006gecode,
  title={Gecode},
  author={Schulte, Christian and Lagerkvist, Mikael and Tack, Guido},
  journal={Software download and online material at the website: http://www. gecode. org},
  pages={11--13},
  year={2006}
}

@article{hebrardmistral,
  title={Mistral, a constraint satisfaction library},
  author={Hebrard, Emmanuel},
  journal={Proceedings of the Third International CSP Solver Competition},
  volume={3},
  number={3},
  pages={31--39},
  year={2008}
}

@inproceedings{kuchcinski2013jacop,
  title={Jacop-java constraint programming solver},
  author={Kuchcinski, Krzysztof and Szymanek, Radoslaw},
  booktitle={CP Solvers: Modeling, Applications, Integration, and Standardization, co-located with the 19th International Conference on Principles and Practice of Constraint Programming},
  year={2013}
}

@inproceedings{yang2021engineering,
  title={Engineering an efficient PB-XOR solver},
  author={Yang, Jiong and Meel, Kuldeep S},
  booktitle={27th International Conference on Principles and Practice of Constraint Programming (CP 2021)},
  year={2021},
  organization={Schloss-Dagstuhl-Leibniz Zentrum f{\"u}r Informatik}
}

@inproceedings{kyrillidis2020fouriersat,
  title={FourierSAT: A Fourier expansion-based algebraic framework for solving hybrid boolean constraints},
  author={Kyrillidis, Anastasios and Shrivastava, Anshumali and Vardi, Moshe and Zhang, Zhiwei},
  booktitle={Proceedings of the AAAI Conference on Artificial Intelligence},
  volume={34},
  number={02},
  pages={1552--1560},
  year={2020}
}

@inproceedings{kyrillidis2021continuous,
  title={On continuous local BDD-based search for hybrid SAT solving},
  author={Kyrillidis, Anastasios and Vardi, Moshe and Zhang, Zhiwei},
  booktitle={Proceedings of the AAAI Conference on Artificial Intelligence},
  volume={35},
  number={5},
  pages={3841--3850},
  year={2021}
}

@article{pote2025towards, 
    title={Towards Real-Time Approximate Counting}, 
    volume={39}, 
    url={https://ojs.aaai.org/index.php/AAAI/article/view/33231}, DOI={10.1609/aaai.v39i11.33231}, 
    number={11}, 
    journal={Proceedings of the AAAI Conference on Artificial Intelligence}, 
    author={Pote, Yash and Meel, Kuldeep S. and Yang, Jiong}, 
    year={2025}, 
    month={Apr.}, 
    pages={11318-11326} }

@book{o2014analysis,
  title={Analysis of boolean functions},
  author={O'Donnell, Ryan},
  year={2014},
  publisher={Cambridge University Press}
}

@article{jain2017non,
  title={Non-convex optimization for machine learning},
  author={Jain, Prateek and Kar, Purushottam},
  journal={arXiv preprint arXiv:1712.07897},
  year={2017}
}

@article{crawford1994minimal,
  title={The minimal disagreement parity problem as a hard satisfiability problem},
  author={Crawford, James M and Kearns, Michael J and Schapire, Robert E},
  journal={Computational Intell. Research Lab and AT\&T Bell Labs TR},
  year={1994},
  publisher={Citeseer}
}

@incollection{pearl1988belief,
  title={Belief updating by network propagation},
  author={Pearl, Judea},
  booktitle={Probabilistic Reasoning in Intelligent Systems},
  pages={143--237},
  year={1988},
  publisher={Morgan Kaufmann San Francisco (CA)}
}

@article{shafer1990probability,
  title={Probability propagation},
  author={Shafer, Glenn R and Shenoy, Prakash P},
  journal={Annals of mathematics and Artificial Intelligence},
  volume={2},
  pages={327--351},
  year={1990},
  publisher={Springer}
}

@article{thornton1994efficient,
  title={Efficient spectral coefficient calculation using circuit output probabilities},
  author={Thornton, MA and Nair, VSS},
  journal={Digital Signal Processing},
  volume={4},
  number={4},
  pages={245--254},
  year={1994},
  publisher={Elsevier}
}

@book{sasao1996representations,
  title={Representations of discrete functions},
  author={Sasao, Tsutomu and Fujita, Masahiro and others},
  year={1996},
  publisher={Springer}
}

@inproceedings{aavani2011translating,
  title={Translating pseudo-boolean constraints into cnf},
  author={Aavani, Amir},
  booktitle={International Conference on Theory and Applications of Satisfiability Testing},
  pages={357--359},
  year={2011},
  organization={Springer}
}

@article{gu1994global,
  title={Global optimization for satisfiability (SAT) problem},
  author={Gu, Jun},
  journal={IEEE Transactions on Knowledge and Data Engineering},
  volume={6},
  number={3},
  pages={361--381},
  year={1994},
  publisher={IEEE}
}

@software{jax2018github,
  author = {James Bradbury and Roy Frostig and Peter Hawkins and Matthew James Johnson and Chris Leary and Dougal Maclaurin and George Necula and Adam Paszke and Jake Vander{P}las and Skye Wanderman-{M}ilne and Qiao Zhang},
  title = {{JAX}: composable transformations of {P}ython+{N}um{P}y programs},
  url = {http://github.com/google/jax},
  version = {0.3.13},
  year = {2018},
}

@article{blondel2022efficient,
  title={Efficient and modular implicit differentiation},
  author={Blondel, Mathieu and Berthet, Quentin and Cuturi, Marco and Frostig, Roy and Hoyer, Stephan and Llinares-L{\'o}pez, Felipe and Pedregosa, Fabian and Vert, Jean-Philippe},
  journal={Advances in neural information processing systems},
  volume={35},
  pages={5230--5242},
  year={2022}
}

@article{vardi2022quantum,
  title={Quantum-inspired perfect matching under vertex-color constraints},
  author={Vardi, Moshe Y and Zhang, Zhiwei},
  journal={arXiv preprint arXiv:2209.13063},
  year={2022}
}

@article{nagy2024ising,
  title={Ising model partition-function computation as a weighted counting problem},
  author={Nagy, Shaan and Paredes, Roger and Dudek, Jeffrey M and Due{\~n}as-Osorio, Leonardo and Vardi, Moshe Y},
  journal={Physical Review E},
  volume={109},
  number={5},
  pages={055301},
  year={2024},
  publisher={APS}
}

@inproceedings{golia2022scalable,
  title={A Scalable Shannon Entropy Estimator},
  author={Golia, Priyanka and Juba, Brendan and Meel, Kuldeep S},
  booktitle={International Conference on Computer Aided Verification},
  pages={363--384},
  year={2022},
  organization={Springer}
}

@inproceedings{bryant1995binary,
  title={Binary decision diagrams and beyond: Enabling technologies for formal verification},
  author={Bryant, Randal E},
  booktitle={Proceedings of IEEE International Conference on Computer Aided Design (ICCAD)},
  pages={236--243},
  year={1995},
  organization={IEEE}
}

@article{li2000integrating,
  title={Integrating equivalency reasoning into Davis-Putnam procedure},
  author={Li, Chu Min},
  journal={AAAI/IAAI},
  volume={2000},
  pages={291--296},
  year={2000}
}

@inproceedings{ignatiev2018pysat,
  title={PySAT: A Python toolkit for prototyping with SAT oracles},
  author={Ignatiev, Alexey and Morgado, Antonio and Marques-Silva, Joao},
  booktitle={International Conference on Theory and Applications of Satisfiability Testing},
  pages={428--437},
  year={2018},
  organization={Springer}
}

@inproceedings{cen2025massively,
  title={Massively parallel continuous local search for hybrid SAT solving on GPUs},
  author={Cen, Yunuo and Zhang, Zhiwei and Fong, Xuanyao},
  booktitle={Proceedings of the AAAI Conference on Artificial Intelligence},
  volume={39},
  number={11},
  pages={11140--11149},
  year={2025}
}

@article{chen2011projection,
  title={Projection onto a simplex},
  author={Chen, Yunmei and Ye, Xiaojing},
  journal={arXiv preprint arXiv:1101.6081},
  year={2011}
}

@book{nesterov2013introductory,
  title={Introductory lectures on convex optimization: A basic course},
  author={Nesterov, Yurii},
  volume={87},
  year={2013},
  publisher={Springer Science \& Business Media}
}

@article{cheng2010mdd,
  title={An MDD-based generalized arc consistency algorithm for positive and negative table constraints and some global constraints},
  author={Cheng, Kenil CK and Yap, Roland HC},
  journal={Constraints},
  volume={15},
  number={2},
  pages={265--304},
  year={2010},
  publisher={Springer}
}

@article{bryant1986graph,
  title={Graph-based algorithms for boolean function manipulation},
  author={Bryant, Randal E},
  journal={Computers, IEEE Transactions on},
  volume={100},
  number={8},
  pages={677--691},
  year={1986},
  publisher={IEEE}
}

@inproceedings{perez2015efficient,
  title={Efficient operations on MDDs for building constraint programming models},
  author={Perez, Guillaume and R{\'e}gin, Jean-Charles},
  booktitle={IJCAI 2015},
  year={2015}
}

@inproceedings{miller2003augmented,
  title={Augmented sifting of multiple-valued decision diagrams},
  author={Miller, D Michael and Drechsler, Rolf},
  booktitle={33rd International Symposium on Multiple-Valued Logic, 2003. Proceedings.},
  pages={375--382},
  year={2003},
  organization={IEEE}
}

@article{allen2014linear,
  title={Linear coupling: An ultimate unification of gradient and mirror descent},
  author={Allen-Zhu, Zeyuan and Orecchia, Lorenzo},
  journal={arXiv preprint arXiv:1407.1537},
  year={2014}
}

@inproceedings{khalil2016learning,
  title={Learning to branch in mixed integer programming},
  author={Khalil, Elias and Le Bodic, Pierre and Song, Le and Nemhauser, George and Dilkina, Bistra},
  booktitle={Proceedings of the AAAI conference on artificial intelligence},
  volume={30},
  number={1},
  year={2016}
}

@article{zhang2022learning,
  title={Learning Branching Heuristics from Graph Neural Networks},
  author={Zhang, Congsong and Gao, Yong and Nastos, James},
  journal={arXiv preprint arXiv:2211.14405},
  year={2022}
}

@article{vo2025learning,
author = {Vo, Trang and Baiou, Mourad and Nguyen, Viet Hung and Weng, Paul},
title = {Learning to Cut Generation in Branch-and-Cut Algorithms for Combinatorial Optimization},
year = {2025},
issue_date = {September 2025},
publisher = {Association for Computing Machinery},
address = {New York, NY, USA},
volume = {5},
number = {3},
issn = {2688-299X},
url = {https://doi.org/10.1145/3728371},
doi = {10.1145/3728371},
journal = {ACM Trans. Evol. Learn. Optim.},
month = aug,
articleno = {20},
numpages = {27}
}

@article{zeng2025beyond,
  title={Beyond Local Selection: Global Cut Selection for Enhanced Mixed-Integer Programming},
  author={Zeng, Shuli and Zhang, Sijia and Li, Shaoang and Wu, Feng and Li, Xiang-Yang},
  journal={arXiv preprint arXiv:2503.15847},
  year={2025}
}

@inproceedings{ling2024learning,
  title={Learning to stop cut generation for efficient mixed-integer linear programming},
  author={Ling, Haotian and Wang, Zhihai and Wang, Jie},
  booktitle={Proceedings of the AAAI Conference on Artificial Intelligence},
  volume={38},
  number={18},
  pages={20759--20767},
  year={2024}
}

@article{gasse2019exact,
  title={Exact combinatorial optimization with graph convolutional neural networks},
  author={Gasse, Maxime and Ch{\'e}telat, Didier and Ferroni, Nicola and Charlin, Laurent and Lodi, Andrea},
  journal={Advances in neural information processing systems},
  volume={32},
  year={2019}
}

@inproceedings{xu2018towards,
  title={Towards Effective Deep Learning for Constraint Satisfaction Problems},
  author={Xu, Hang and Koenig, Sven and Kumar, T. K. Satish},
  booktitle={Principles and Practice of Constraint Programming},
  editor={Hooker, John},
  series={Lecture Notes in Computer Science},
  volume={11008},
  pages={588--597},
  year={2018},
  publisher={Springer},
  doi={10.1007/978-3-319-98334-9_38}
}

@inproceedings{goyal2024deepsade,
  title={Deepsade: Learning neural networks that guarantee domain constraint satisfaction},
  author={Goyal, Kshitij and Dumancic, Sebastijan and Blockeel, Hendrik},
  booktitle={Proceedings of the AAAI Conference on Artificial Intelligence},
  volume={38},
  number={11},
  pages={12199--12207},
  year={2024}
}

@inproceedings{
xu2025selfsupervised,
title={Self-Supervised Transformers as Iterative Solution Improvers for Constraint Satisfaction},
author={Yudong Xu and Wenhao Li and Scott Sanner and Elias Boutros Khalil},
booktitle={Forty-second International Conference on Machine Learning},
year={2025},
url={https://openreview.net/forum?id=IQN6ID0snT}
}

@article{kyrillidis2021solving,
  title={Solving hybrid Boolean constraints in continuous space via multilinear Fourier expansions},
  author={Kyrillidis, Anastasios and Shrivastava, Anshumali and Vardi, Moshe Y and Zhang, Zhiwei},
  journal={Artificial Intelligence},
  volume={299},
  pages={103559},
  year={2021},
  publisher={Elsevier}
}
\bibliographystyle{icml2026}

\newpage
\appendix
\onecolumn
\section{Technical Proof}\label{sec:proof}
\subsection{Proof of Theorem~\ref{thm:expectation}}
Let $P\in\tilde{D}$ and $X\in{}D$. 
\begin{align}
    \mathop{\mathbb{E}}_{X\sim\mathcal{S}_P}[\mathcal{F}(X)]
    & =\mathop{\mathbb{E}}_{X\sim\mathcal{S}_P}[\sum_{c\in{}C}f_c(X)] \nonumber\\
    & =\sum_{c\in{}C}\mathop{\mathbb{E}}_{X\sim\mathcal{S}_P}[f_c(X)] \nonumber\\
    & =\sum_{\alpha\in{}D} \hat{\mathcal{F}}(\alpha) \cdot \mathop{\mathbb{E}}_{X\sim\mathcal{S}_P}[\phi_\alpha(X)]
    \label{eq:appx:exp1}
\end{align}

Since $\phi_\alpha(X)\in\{0,1\}$ By linearity of the expectation, we have
\begin{align}
    \mathop{\mathbb{E}}_{X\sim\mathcal{S}_P}[\phi_\alpha(X)] 
    &= 1\cdot\mathop{\mathbb{P}}_{X\sim\mathcal{S}_P}[\phi_\alpha(X)=1] + 0\cdot\mathop{\mathbb{P}}_{X\sim\mathcal{S}_P}[\phi_\alpha(X)=0] \nonumber\\
    &= \mathop{\mathbb{P}}_{X\sim\mathcal{S}_P}[\phi_\alpha(X)=1] \nonumber\\
    &= \mathop{\mathbb{P}}_{X\sim\mathcal{S}_P}\left[\prod_{i=1}^n\phi_{\alpha_i}(x_i)=1\right].
    \label{eq:appx:exp2}
\end{align}

As $\phi_{\alpha_i}(\cdot)$'s are indication function, 
\begin{align}
    \mathop{\mathbb{P}}_{X\sim\mathcal{S}_P}\left[\prod_{i=1}^n\phi_{\alpha_i}(x_i)=1\right]
    &= \prod_{i=1}^n\mathop{\mathbb{P}}_{X\sim\mathcal{S}_P}[\phi_{\alpha_i}(x_i)=1] \nonumber\\
    &= \prod_{i=1}^n\mathbb{P}[\mathcal{R}(P)_i=x_i] \nonumber\\
    &= \prod_{i=1}^np_{i,x_i}
    \label{eq:appx:exp3}
\end{align}
where the second equation holds due to Definition~\ref{def:rounding}. 
By substituting Eq.~\eqref{eq:appx:exp2} and Eq.~\eqref{eq:appx:exp3} into Eq.~\eqref{eq:appx:exp1}, we can have
\begin{align*}
    \mathop{\mathbb{E}}_{X\sim\mathcal{S}_P}[\mathcal{F}(X)]
    =\sum_{\alpha\in{}D} \hat{\mathcal{F}}(\alpha) \cdot \prod_{i=1}^np_{i,x_i}
    \label{eq:appx:exp1}
\end{align*}

\subsection{Proof of Proposition~\ref{prop:lipschitz}}
By the triangle inequality, 
\begin{align*}
        & |\mathcal{F}(P)-\mathcal{F}(P')| \\
    \le & |\mathcal{F}(P_{1,1},P_{1,2},\cdots,P_{n,|\mathrm{dom}(x_n)|})-\mathcal{F}(P'_{1,1},P_{1,2},\cdots,P_{n,|\mathrm{dom}(x_n)|})|\\
    +   & |\mathcal{F}(P'_{1,1},P_{1,2},\cdots,y_{n,|\mathrm{dom}(x_n)|})-\mathcal{F}(P'_{1,1},P'_{1,2},\cdots,P_{n,|\mathrm{dom}(x_n)|})| + \cdots \\
    +   & |\mathcal{F}(P'_{1,1},P'_{1,2},\cdots,y_{n,|\mathrm{dom}(x_n)|})-\mathcal{F}(P'_{1,1},P'_{1,2},\cdots,P'_{n,|\mathrm{dom}(x_n)|})| \\
    =   & |(P_{1,1}-P'_{1,1}) \cdot \partial_{1,1}\mathcal{F}(P'_{1,2},\cdots,P'_{n,|\mathrm{dom}(x_n)|})| + \cdots \\
    +   & |(P_{n,|\mathrm{dom}(x_n)|}-P'_{n,|\mathrm{dom}(x_n)|}) \cdot \partial_{n,|\mathrm{dom}(x_n)|}\mathcal{F}(P'_{1,1},\cdots,P'_{n,|\mathrm{dom}(x_n)|-1})|\\
\end{align*}
Since $\mathcal{F}(P)$ is a multilinear function, $|\partial_{i,j}\mathcal{F}|\le|C|$ for all $i\in[n]$ and $j\in[|\mathrm{dom}(x_i)|]$.
Hence we have
\begin{align*}
        & |(P_{1,1}-P'_{1,1})\partial_{1,1}\mathcal{F}(P_{1,2},\cdots,P_{n,|\mathrm{dom}(x_n)|})| + \cdots \\
    + &  |(P_{n,|\mathrm{dom}(x_n)|}-P'_{n,|\mathrm{dom}(x_n)|})\partial_{n,|\mathrm{dom}(x_n)|}\mathcal{F}(P'_{1,1},\cdots,P'_{n,|\mathrm{dom}(x_n)|-1})|\\
    \le & |C|\left(\sum_{i\in[n]}\sum_{j\in[|\mathrm{dom}(x_i)|]}|P_{i,j}-P'_{i,j}|\right)\\
    \le & |C|\sqrt{N}\|P-P'\|
\end{align*}
The last inequality is an application of the Cauchy-Schwarz Inequality.


\subsection{Proof of Theorem~\ref{thm:reduction}}
\begin{itemize}
    \item ``$\Rightarrow$'': 
    Suppose the Boolean formula $F$ is satisfiable and $X\in D$ is one of the solutions.
    Then a real assignment $P\in\tilde{D}$ such that $\mathbb{P}(\mathcal{R}(P)=X)=1$ we have $\mathop{\mathbb{E}}_{X\sim\mathcal{S}_P}[f_c(X)]=1$ for $c\in{C}$.
    Therefore $\mathcal{F}(X)=|C|$ and $\max \mathop{\mathbb{E}}_{X\sim\mathcal{S}_P}[\mathcal{F}(X)] = |C|$.
    \item ``$\Leftarrow$'': 
    Suppose Eq.~\eqref{eq:reduction} holds.
    For every $c\in{}C$ we have $\mathbb{E}_{X\sim{}\mathcal{S}_P}[f_c(X)]=1$.
    Thus, the rounding $X\sim{}\mathcal{S}_P$ is a discrete assignment that satisfies all constraints.
\end{itemize}

\subsection{Proof of Lemma~\ref{lmm:saddle}}
Since $\mathcal{F}$ is a multilinear function, it can be decomposed as
\begin{equation*}
    \mathcal{F}(P) = \sum_{j\in\mathrm{dom}(x_1)}p_{1,j}{}g_j(P_{\backslash{}1}) + h(P_{\backslash{}1}),
\end{equation*}
where $g_j$'s and $h$ are multilinear polynomials not depending on $j$.
Let
\begin{align*}
    \eta q_{i} = \begin{cases}
        {{proj}_{\tilde{D}_i}(p_1 - \eta \nabla \mathcal{F}(p_1)) - p_1}, & i=1,\\
        \vec{0}, & otherwise,
    \end{cases}
\end{align*}
and $\epsilon$ be an arbitrary positive number.
For all $\eta$ such that $\eta\in(0,\epsilon)$, we have
\begin{align*}
    &\mathcal{F}(P-\eta{}V) - \mathcal{F}(P) \\
    =& \sum_{j\in\mathrm{dom}(x_1)}\eta v_{1,j} \cdot g_j(P_{\backslash{}1}) \\
    =& \sum_{j\in\mathrm{dom}(x_1)}\left({proj}_{\tilde{D}}(p_1 - \eta \nabla \mathcal{F}(p_1))_j - p_{1,j}\right) \cdot g_j(P_{\backslash{}1}).
\end{align*}
If ${proj}_{\Delta_1}(p_1-\eta{}q_1)$ is on $\Delta\backslash{\delta\Delta}$, we have
\begin{align*}
    &\frac{1}{\eta}\left(\mathcal{F}(P-\eta{}V) - \mathcal{F}(P)\right)\\
    =& \frac{1}{\eta}\sum_{j\in\mathrm{dom}(x_1)}\left(\eta g_j(P_{\backslash{}1})
    -\frac{\sum_{j\in\mathrm{dom}(x_1)}\left( p_{1,j} + \eta g_j(P_{\backslash{}1})\right)-1}{|\mathrm{dom}(x_1)|}\right)g_j(P_{\backslash{}1})\\
    =& \sum_{j\in\mathrm{dom}(x_1)}\left(g_j(P_{\backslash{}1})
    -\frac{\sum_{j\in\mathrm{dom}(x_1)}g_j(P_{\backslash{}1})}{|\mathrm{dom}(x_1)|}\right)\cdot g_j(P_{\backslash{}1})\\
    =& \sum_{j\in\mathrm{dom}(x_1)}g_j(P_{\backslash{}1})^2 - \frac{1}{|\mathrm{dom}(x_1)|}\left(\sum_{j\in\mathrm{dom}(x_1)}g_j(P_{\backslash{}1})\right)^2\\
    =& \left(\sum_{j\in\mathrm{dom}(x_1)}g_j(P_{\backslash{}1})^2\right)\left(\sum_{j\in\mathrm{dom}(x_1)}\frac{1}{|\mathrm{dom}(x_1)|}\right)
    - \left(\sum_{j\in\mathrm{dom}(x_1)}\frac{g_j(P_{\backslash{}1})}{\sqrt{|\mathrm{dom}(x_1)|}}\right)^2\\
    \ge & 0
\end{align*}
The last inequality is the Cauchy-Schwarz inequality, where the equality holds when $\mathcal{F}$ is a constant.

\subsection{Proof of Theorem~\ref{thm:optimality}}
From Lemma~\ref{lmm:saddle} we can observe that, if ${proj}_{\tilde{D}}(P-\eta{}V)$ is on ${\delta\tilde{D}}$, we can always find an $\epsilon'\in(0,\epsilon)$ such that for all $\eta'\in(0,\epsilon')$, ${proj}_{\tilde{D}}(P-\eta'V)$ is on ${\tilde{D}\backslash\delta\tilde{D}}$.

Lemma~\ref{lmm:saddle} states that, any interior point $a\in\tilde{D}\backslash{}\delta\tilde{D}$ cannot be a local maximum of $\mathcal{C}(a)$.
In contrast, all local maxima must lie on the boundary $\delta\tilde{D}$.
Hence, Lemma~\ref{lmm:saddle} directly results in Theorem~\ref{thm:optimality}.

\subsection{Proof of Theorem~\ref{thm:pgd}}

To prove Theorem~\ref{thm:pgd}, we first need to establish Lemma~\ref{lmm:pgd}.
The proof of Lemma~\ref{lmm:pgd} relies on three propositions.

\begin{proposition}[Projection Point~\cite{nesterov2013introductory}]\label{prop:proj_point}
    For $P, P'\in\tilde{D}$, let
    $h(P)=\|P-(P'+\eta\nabla\mathcal{F}(P'))\|^2$
    and
    $P^*=\argmin_{P\in\tilde{D}}h(P)$, we have
    \begin{equation*}
        \left<
            \nabla{}h(P^*), P-P^*
        \right>\ge 0.
    \end{equation*}
\end{proposition}

Note that the projection point $P^*$ is $P'-\eta{}g(P')$; 
Hence, we have the following Proposition.

\begin{proposition}[Gradient Mapping~\cite{kyrillidis2021solving}]\label{prop:grad_map}
    Given a function $\mathcal{F}$ defined on a convex set $\Omega$, for every $P\in\Omega$, we have
    \begin{equation*}
        \left<\nabla\mathcal{F}(P), g(P)\right>\ge\|g(P)\|^2,
    \end{equation*}
    where $g(P)$ is the gradient mapping.
\end{proposition}

Also, the proof relies on the consequence of Proposition~\ref{prop:lipschitz}.
By representing $\nabla\mathcal{F}(y)-\nabla\mathcal{F}(z)$ as an integral, we have
\begin{align}
    \nonumber& \left|\mathcal{F}(P) - \mathcal{F}(P') - \left<\nabla\mathcal{F}(P'),(P-P')\right>\right| \\
    \nonumber= & \left|\int_0^1 \left<\nabla\mathcal{F}(P'+t(P-P')),(P-P')\right>dt - \left<\nabla\mathcal{F}(P'),(P-P')\right>\right|\\
    \nonumber\le& \int_0^1 \|\nabla\mathcal{F}(P'+t(P-P')) - \nabla\mathcal{F}(P')\|\cdot \|P-P'\|dt \\
    \nonumber\le& \int_0^1 Lt\|P-P'\|^2dt \\
    = & \frac{L}{2}\|P-P'\|^2dt
    \label{eq:lipschitz}
\end{align}

\begin{lemma}\label{lmm:pgd}
    Let $\mathcal{F}$ be a multilinear polynomial.
    Let $n$ be the number of variable, then $N=\sum_{i=1}^n|\mathrm{dom}(x_i)|$.
    Let $g(\cdot)$ be the projected gradient mapping onto the simplices $\tilde{D}$.
    Given $P, P'\in\tilde{D}$, if the step size satisfies $\eta\le\frac{1}{L}$, then the projected gradient step $P=P'-\eta{}g(P')$ satisfies:
    \begin{equation*}
        \mathcal{F}(P) - \mathcal{F}(P') \ge \frac{\eta}{2}\|g(P')\|^2
    \end{equation*}
\end{lemma}

\begin{proof}[Proof of Lemma~\ref{lmm:pgd}]
    Eq.~\eqref{eq:lipschitz} can be rewritten as 
    \begin{equation*}
        \mathcal{F}(P) \le \mathcal{C}(P') + \left<
            \nabla\mathcal{F}(P'), P-P'
        \right> + \frac{L}{2}\|P-P'\|^2 
    \end{equation*}
    Let $P=P'-\eta{}g(P')$, the above inequality becomes
    \begin{equation*}
        \mathcal{F}(P) \le \mathcal{F}(P') - \eta\left<
            \nabla\mathcal{F}(P'), g(P')
        \right> + \frac{\eta^2L}{2}\|g(P')\|^2.
    \end{equation*}
    Applying Proposition~\ref{prop:grad_map} and let $\eta\le{}\frac{1}{L}$, we then have
    \begin{align*}
        \mathcal{F}(P) 
        &\le \mathcal{F}(P') - \eta\|g(P')\|^2 + \frac{\eta}{2}\|g(P')\|^2 \\
        &= \mathcal{F}(P') - \frac{\eta}{2}\|g(P')\|^2.
    \end{align*}
\end{proof}

From Lemma~\ref{lmm:pgd}, we can assume $\eta=\frac{1}{L}$. 
Consider a sequence of points $P_{t}$'s such that $P_{t+1} = P_{t} - \eta g(P_{t})$.
Due to Lemma~\ref{lmm:pgd} we have $\mathcal{F}(P_{t+1})-\mathcal{F}(P_{t})\le-\frac{\eta}{2}\|g(P_{t})\|^2$.
Given any initial point $P_{0}$, it will converge to $P^*$. 
After $T$ gradient steps, we have
\begin{align*}
    \mathcal{F}(P^*) - \mathcal{F}(P_{(0)})
    & \ge \frac{1}{2L}\sum_{t=0}^T \|g(P_{t})\|^2\\
    & \ge \frac{T+1}{2L}\min_{t} \|g(P_{t})\|^2,
\end{align*}
which can be rewritten as
\begin{align*}
    \min_{t} \|g(P_{t})\| 
    &\le \sqrt{\frac{2L}{T+1}} \sqrt{\mathcal{F}(P^*) - \mathcal{C}(P_{0})}\\
    &\le \sqrt{\frac{2L}{T+1}} \sqrt{|C|}
\end{align*}
To achieve a point $a_{T}$ such that $\|g(P_{t})\|\le\epsilon$, we require $\sqrt{\frac{2|C|L}{T+1}}\le\epsilon$, which implies $T\sim \mathcal{O}\left(\frac{|C|\cdot L}{\epsilon^2}\right)$.

\subsection{Proof of Theorem~\ref{thm:md}}
We refer to the full convergence proof in Appendix B.2 of~\cite{allen2014linear}.
$\rho$ was shown in Proposition~\ref{prop:lipschitz}
Here, we justify the expression of $\Theta=\sum_{i=1}^n\log|\mathrm{dom}(x_i)|$ as the Bregman diameter under the negative-entropy potential, \emph{i.e.}, $\psi(P)=\langle P,\log P\rangle$.
Hence the quantity $\Theta$ is defined as 
\begin{equation*}
    \Theta = \sup_P \left(\psi(P)-\psi(P_0)\right),
\end{equation*}
where we choose $P_0$ to be the uniform distribution on each probability simplex, since this minimizes $\psi(P_0)$ and therefore yields the tightest bound. 
Hence, 
\begin{align*}
    \psi(P) &= \sum_{i=1}^n\sum_{j=1}^{|\mathrm{dom}(x_i)|}p_{i,j}\log p_{i,j}\\
    \psi(P_0) &= \sum_{i=1}^n\sum_{j=1}^{|\mathrm{dom}(x_i)|}\frac{1}{|\mathrm{dom}(x_i)|}\log \frac{1}{|\mathrm{dom}(x_i)|}
\end{align*}
\begin{align*}
    \Theta &= \sup_P \left(\psi(P)-\psi(P_0)\right)\\
    &= \max_P\left(\psi(P)\right) - \psi(P_0) \\
    &= \sum_{i=1}^n \max_{p_i}\left(\sum_{j=1}^{|\mathrm{dom}(x_i)|}p_{i,j}\log p_{i,j}\right) + \sum_{i=1}^n\log|\mathrm{dom}(x_i)|\\
    &= \sum_{i=1}^n \log|\mathrm{dom}(x_i)|
\end{align*}

\section{Benchmarks and Encodings}\label{sec:benchmark_encode}
\subsection{Task Scheduling Problem}\label{sec:task}
Task scheduling problems can be useful in compilers for multi-core CPUs.
Given a task graph $G(V,E)$, where $V\ni{}v$ is a set tasks, $E\ni(u\to{}v)$ is a set of task dependency.
Given $T$ clock cycles and $S$ workers, we define variables $t_v\in\{0, \cdots, T-1\}$ and $s_v\in\{0, \cdots, S-1\}$ to denote the task $v$ is execute at $t_v$ by worker $s_v$.
Then, the scheduling problem can be formulated as 
\begin{align*}
    F &= F_1\wedge F_2,  \\
    F_1 &= \bigwedge_{(u\to{}v)\in{}E}(t_u<t_v),  \\
    F_2 &= \bigwedge_{(u\to{}v)\notin{}E}\left((t_u\ne{}t_v)\vee(s_u\ne{}s_v)\right).
\end{align*}

For each $T\in\{32, 64, 128, 256, 512\}$, $S\in\{4, 8, 16\}$, we choose $|V|={T\cdot S}/{2}$ to generate 10 directed graph $G(V,E)$.
Every variable $v\in{}V$ has a variable index $idx(v)$.
Between any two variable $u,v\in{}V$ and $idx(v)>idx(u)$, we add an edge $u\to{}v$ with probability $5^{idx(u)-idx(v)}$.

\paragraph{Circuit-Output Probability}
For any $(u\to{}v)\in{}E$, the corresponding Fourier expansion, \emph{i.e.}, circuit-output probability of the constraint $c\in{}F_1$ can be written as
\begin{align*}
    \mathbb{P}[t_u<t_v] 
        &= 1-\mathbb{P}[t_u\ge{}t_v] \\
        &= 1-\sum_{i=0}^{T-1}\sum_{j=i+1}^{T-1}\mathbb{P}[t_u\ge{}i]\cdot\mathbb{P}[t_v=i],
\end{align*}
where $\mathbb{P}[t_u\ge{}i]$ can be efficiently computed by cumulative sum.

For any $(u\to{}v)\notin{}E$, the corresponding Fourier expansion, \emph{i.e.}, circuit-output probability of the constraint $c\in{}F_2$ can be written as
\begin{align*}
    \mathbb{P}[(t_u\ne{}t_v)\vee(s_u\ne{}s_v)] 
    =& 1-\mathbb{P}[(t_u=t_v)\wedge(s_u=s_v)]\\
    =& 1-\mathbb{P}[(t_u=t_v)]\cdot\mathbb{P}[(s_u=s_v)].
\end{align*}

\paragraph{Linearization}

While encoding $F_1$ into {ILP} is straightforward, encoding $F_2$ requires auxiliary variables $q_{u,v,i}\in\{0,1\}$. 
\begin{align*}
    F^{int} &= F_1^{int}\wedge F_2^{int}\wedge F_3^{int}\wedge F_4^{int}\wedge F_4^{int}\wedge F_6^{int},  \\
    F_1^{int} &= \bigwedge_{(u\to{}v)\in{}E}(t_v-t_u\ge{}1),  \\
    F_2^{int} &= \bigwedge_{(u\to{}v)\notin{}E}(t_u - t_v + Tq_{u,v,1} \ge 1), \\
    F_3^{int} &= \bigwedge_{(u\to{}v)\notin{}E}(t_v - t_u + Tq_{u,v,2} \ge 1), \\
    F_4^{int} &= \bigwedge_{(u\to{}v)\notin{}E}(s_u - s_v + Sq_{u,v,3} \ge 1), \\
    F_5^{int} &= \bigwedge_{(u\to{}v)\notin{}E}(s_v - s_u + Sq_{u,v,4} \ge 1), \\
    F_6^{int} &= \bigwedge_{(u\to{}v)\notin{}E}\left(\sum_{i=1}^4q_{u,v,i} \le 3\right). \\
\end{align*}

For $F_2^{int}$ to $F_6^{int}$, the interpretation is as follows: given $(u\to{}v)\notin{}E$, when $(t_u=t_v)\wedge(s_u=s_v)$, all $q_{u,v,i}$ should be 1 to satisfy the constraints in $F_2^{int}$ to $F_5^{int}$.
However, the constraint in $F_6^{int}$ is violated.
Only when $(t_u\ne{}t_v)\vee(s_u\ne{}s_v)$, at least one of the $q_{u,v,i}$ can be 0, then the constraint in $F_6^{int}$ can be satisfied.

\paragraph{Booleanization}

To further Booleanize this problem into Boolean constraints, we use a log-2 encoding.
\begin{align}
    t_v &= \sum_{i=0}^{\log_2(T)-1} 2^i t_{v,i}, \label{eq:log2t}\\
    s_v &= \sum_{i=0}^{\log_2(S)-1} 2^i s_{v,i}, \label{eq:log2s} 
\end{align}
where $t_{v,i}\in\{0,1\}$ and $s_{v,i}\in\{0,1\}$. 
Similarly, substituting $F_1^{int}$ with Eq.~\eqref{eq:log2t} directly leads to $F_1^{bool}$.
And for encoding $F_2$, we introduce $q_{u,v,i}^{(t)}\in\{0,1\}$ and $q_{u,v,i}^{(s)}\in\{0,1\}$.

\begin{align*}
    F^{bool} &= F_1^{bool}\wedge F_2^{bool}\wedge F_3^{bool},  \\
    F_1^{bool} &= \bigwedge_{(u\to{}v)\in{}E}\left(\sum_{i=0}^{\log_2(T)-1} 2^i t_{v,i}-\sum_{i=0}^{\log_2(T)-1} 2^i t_{u,i}\ge{}1\right),  \\
    F_2^{bool} &= \bigwedge_{(u\to{}v)\notin{}E}\bigwedge_{i=0}^{\log_2(T)-1}(t_{u,i} \oplus t_{v,i} \oplus \neg{}q_{u,v,i}^{(t)}), \\
    F_3^{bool} &= \bigwedge_{(u\to{}v)\notin{}E}\bigwedge_{i=0}^{\log_2(S)-1}(s_{u,i} \oplus s_{v,i} \oplus \neg{}q_{u,v,i}^{(s)}), \\
    F_4^{bool} &= \bigwedge_{(u\to{}v)\notin{}E}\left(\bigvee_{i=0}^{\log_2(T)-1}q_{u,v,i}^{(t)} \vee \bigvee_{i=0}^{\log_2(S)-1}q_{u,v,i}^{(s)}\right). \\
\end{align*}

$F_2^{bool}$ and $F_3^{bool}$ (resp. $F_4^{bool}$) should be interpreted similarly to $F_2^{int}$ to $F_5^{int}$ (resp. $F_6^{int}$).
Any $t_{u,i}\ne t_{v,i}$ will lead to $q_{u,v,i}^{(t)}=1$ (or $s$ equivalently), the constraint in $F_4^{bool}$ will be satisfied.

\subsection{Graph Coloring Problem with Hashing Queries}~\label{sec:color}
Given a graph $G(V,E)$, where $V\ni{}v$ is a vertex set, $E\ni(u, v)$ is an edge set.
Given $C$ colors, we define variables $x_v\in\{0, \cdots, C-1\}$ to denote the vertex $v$ is colored by $x_v$.

Then, the graph coloring problem can be formulated as a constraint set
\begin{align*}
    C_1 &= \bigwedge_{(u,v)\in{}E}(x_u\ne{}x_v). 
\end{align*}

In this problem, we additionally use a hash function to generate soft modulo constraints.
The hash code $h\in{}H$ randomly samples half of the variables from the vertex set. 
The hash function is defined by $\sum_{i\in h} x_i\mod 2$. 
Then, the soft constraint set is:
\begin{align*}
    C_2 &= \bigwedge_{h\in{}H}\left(\sum_{i\in h} x_i\mod 2=1\right). 
\end{align*}

The constraint optimization problem is
\begin{align*}
    \min_a |H|\cdot\text{cost}(C_1, a) + \text{cost}(C_2, a), 
\end{align*}
where $\text{cost}(C, a)$ is the number of unsatisfied constraints in $C$ when given an assignment $a$.

For each $|V|\in\{512, 1024, 2048, 4096\}$, $C\in\{8, 16, 32, 64\}$, we generate 10 random regular graphs with degree $2C$.
We add $|V|$ modulo constraints as soft constraints. 

\paragraph{Circuit-Output Probability}
For any $(u,v)\in{}E$, the corresponding Fourier expansion, \emph{i.e.}, circuit-output probability of the constraint $c\in{}C_1$ can be written as
\begin{align*}
    \mathbb{P}[x_u\ne{}x_v] 
        &= 1-\mathbb{P}[x_u=x_v] \\
        &= 1-\mathbb{P}[x_u=i]\cdot\mathbb{P}[x_v=i].
\end{align*}

For any $h\in{}H$, as inspired by the Fourier expansion of XOR constraints~\cite{kyrillidis2020fouriersat}, the circuit-output probability of the constraint $c\in{}C_2$ can be written as
\begin{align*}
    \mathbb{P}\left[\sum_{i\in{}h}x_i\mod{2}=1\right] 
    =& \mathbb{P}[x_j\mod{2}=0]\cdot\mathbb{P}\left[\sum_{i\in{}h\backslash{j}}x_i\mod{2}=1\right] \\
    +& \mathbb{P}[x_j\mod{2}=1]\cdot\mathbb{P}\left[\sum_{i\in{}h\backslash{j}}x_i\mod{2}=0\right] \\
    =& \frac{1-\left(1-2\mathbb{P}[x_j\mod{2}=1]\right)\cdot\left(1-2\mathbb{P}\left[\sum_{i\in{}h\backslash{j}}x_i\mod{2}=1\right]\right)}{2} \\
    =& \frac{1-\prod_{i\in{}h}\left(1-2\sum_{j=0}^{\frac{C}{2}-1}\mathbb{P}[x_i=2j+1]\right)}{2},
\end{align*}
which can be handled by slicing.

After the above relaxations, the main computation of \textsc{FourierCSP} can be amenable for parallelization using Jax~\cite{jax2018github}. 

\paragraph{Linearization}
When encoding $C_1$ for {ILP}, we introduce two auxiliary variables $q_{u,v}, q_{v,u}\in\{0,1\}$ for every $(u,v)\in{}E$.
\begin{align*}
    C_1^{int} &= \bigwedge_{(u,v)\in{}E}\left(x_u-x_v+Cq_{u,v}\ge{}1\right),\\
    C_2^{int} &= \bigwedge_{(u,v)\in{}E}\left(x_v-x_u+Cq_{v,u}\ge{}1\right),\\
    C_3^{int} &= \bigwedge_{(u,v)\in{}E}\left(q_{u,v}+q_{v,u}\le{}1\right).
\end{align*}

When encoding $C_2$ for {ILP}, we introduce two auxiliary variables $q_{h}\in\left\{0, ..., \frac{|h|(C-1)}{2}\right\}$, and $p_{h}\in\{0,1\}$ for every $h\in{}H$.
\begin{align*}
    C_4^{int} &= \bigwedge_{h\in{}H}\left(\sum_{i\in h} x_i - 2q_{h} + p_{h} = 0\right), 
\end{align*}
where $p_h$ should be 1 (resp. 0) when $\sum_{i\in h} x_i$ is odd (resp. even).
Then {ILP} should be formulated as:
\begin{align*}
    & \min |H| - \sum_{h\in{}H}p_h \\
    & \text{s.t. } C_1^{int}\cup C_2^{int}\cup C_3^{int} \cup C_4^{int}
\end{align*}

\paragraph{Booleanization}
The Booleanization for PB-XOR is similar to the previous section. 
\begin{align}
    x_v &= \sum_{i=0}^{\log_2(C)-1} 2^i x_{v,i}, \label{eq:log2x}
\end{align}
where $x_{v,i}\in\{0,1\}$. 

For every $(u,v)\in{}E$, and $i\in\{0,...,\log_2(C)-1\}$, we introduce an auxiliary variable $q_{u,v,i}\in\{0,1\}$.
Consequently, $C_1$ can be encoded by
\begin{align*}
    C_1^{bool} &= \bigwedge_{(u,v)\in{}E}(x_{u,i}\oplus{}x_{v,i}\oplus{}\neg{}q_{u,v,i}), \\
    C_2^{bool} &= \bigwedge_{(u,v)\in{}E}\left(\bigvee_{i=0}^{\log_2(C)-1} q_{u,v,i}\right).
\end{align*}

Note that the parity of $x_u$ only depends on $x_{u,0}$.
Hence, the soft constraint set $C_2$ can be encoded by
\begin{equation*}
    C_3^{bool} = \bigwedge_{h\in{}H}\left(\bigoplus_{i\in{}h}x_{i,0}\right)
\end{equation*}

Eventually, the constrained optimization problem becomes
\begin{equation*}
    \min_a |H|\cdot\left(\text{cost}(C_1^{bool}, a)+\text{cost}(C_2^{bool}, a)\right) + \text{cost}(C_3^{bool}, a).
\end{equation*}

\section{Efficient Evaluation and Differentiation by Decision Diagrams}\label{sec:dd}
Theorem~\ref{thm:reduction} reveals that \ac{CSP} solving can be reduced to a continuous optimization problem.
The performance of gradient-based optimization crucially depends on the efficiency of evaluating and differentiating the objective function.
However, the complexity of evaluation and differentiation is in \#P-hard, \emph{i.e.}, Eq.~\eqref{eq:expectation} has at most exponentially many terms, in general.
Motivated by~\cite{thornton1994efficient}, we reformulate Eq.~\eqref{eq:expectation} to \ac{COP}, which admits more tractable computation when a structured representation can encode the underlying constraint, \emph{i.e.}, decision diagrams.

\begin{definition}[Circuit-Output Probability]
    Let $f_c$ be a discrete function over a set $X$ of finite domain variables.
    Let $P$ be the variable input probabilities.
    The circuit-output probability problem of $f_c$ w.r.t $P$, denoted by ${COP}_c(P)$, is the probability of $f_c$ outputs $1$ given the value of each variable independently sampled from a multimodal distribution, \emph{i.e.},
    \begin{equation*}
        {COP}_c(P) = \sum_{X\in D} f_c(X)\prod_{i=1}^n p_{i,x_i}. 
    \end{equation*}
\end{definition}

\begin{corollary}\label{cor:COP}
    Given a constraint $c$ defined on a finite domain $D$, for a real point $P\in\tilde{D}$, we have
    \begin{align*}
        \mathop{\mathbb{E}}_{X\sim\mathcal{S}_P}[f_c(X)] = {COP}_c(P).
    \end{align*}
\end{corollary}

\begin{proof}[Proof of Corollary~\ref{cor:COP}]
Let $P\in\tilde{D}$ and $X\in{}D$. 
For each $c\in{}C$, by linearity of expectation, we have
\begin{align*}
    \mathop{\mathbb{E}}_{X\sim\mathcal{S}_P}[f_c(X)]
    &= 1\cdot\mathop{\mathbb{P}}_{X\sim\mathcal{S}_P}[f_c(X)=1] + 0\cdot\mathop{\mathbb{P}}_{X\sim\mathcal{S}_P}[f_c(X)=0]\\
    &= \mathop{\mathbb{P}}_{X\sim\mathcal{S}_P}[f_c(X)=1]\\
    &= {COP}_c(P).
\end{align*}
\end{proof}

Prior work demonstrates that BDD can encode many Boolean constraints in compact size~\cite{kyrillidis2021continuous,sasao1996representations,aavani2011translating}.
The multiple-valued extension of BDD is called a multi-valued decision diagram (MDD)~\cite{bryant1995binary}. 
The MDD can be versatile in representing many constraints through the translation from table constraints~\cite{cheng2010mdd}\footnote{A table constraint is explicitly represented as the set of solutions or non-solutions for that constraint.}.
Similar to BDD, two MDDs can be combined through an \textsc{Apply} operation~\cite{bryant1986graph,perez2015efficient}, offering flexibility in encoding versatile constraints.
The MDDs can be reduced by merging nodes that have the same set of outgoing neighbors associated with the same label~\cite{cheng2010mdd,perez2015efficient}.
Moreover, the size can be further reduced by sifting~\cite{miller2003augmented}.
A compact-size decision diagram means that the COP on it can be evaluated and differentiated efficiently.

\subsection{Evaluation}
\begin{lemma}[Top-Down Evaluation]\label{lmm:forward}
    Let $G(V,E)$ be the MDD of constraint $c$, run Algorithm~\ref{alg:forward} on $G$ with a real assignment $P\in\tilde{D}$.
    Let $X\in D$ randomly rounded from $P$ by $\mathbb{P}[x_i=j]=p_{i,j}$ for all $j\in\mathrm{dom}(x_i)$.
    For each $v\in V$, we have
    \begin{equation*}
        m_{td}[v] = \mathbb{P}_{M,a}[v],
    \end{equation*}
    where $\mathbb{P}_{G,a}[v]$ is the probability that the node $v$ is on the path generated by $b$ on $G$.
    Especially, 
    \begin{equation*}
        m_{td}[\textsc{true}] = {COP}_c(P_a),
    \end{equation*}
    \emph{i.e.}, Algorithm~\ref{alg:forward} returns the COP of $G$ in time $O(|V|)$.
\end{lemma}

\begin{proof}[Proof of Lemma~\ref{lmm:forward}]
We prove by induction on the MDD $G(V, E)$.
Let $A\in\tilde{D}$ and $B\in{}D$. 
\begin{itemize}
    \item \emph{Basic step}: For the root node $r$ of $M$, $m_{td}[r]=1$. This equation holds because $r$ is on the path generated by every discrete assignment on $G$.

    \item \emph{Inductive step}: For each non-root node $v$ of $V$, let $par_i(v)$ be ths set of parent nodes of $v$ with an edge labeled by $i$.
    After Algorithm~\ref{alg:forward} terminate, we have
    \begin{equation*}
        m_{td}[v] = \sum_{u\in par_i(v)} a_{id(u), i}\cdot m_{td}[u],
    \end{equation*}
    where $id(u)$ is the mapping of the node $u$ to the corresponding variable index.
    By inductive hypothesis,
    \begin{align*}
        m_{td}[v] 
        &= \sum_{u\in par_i(v)} \mathbb{P}[b_{id(u)}=i] \cdot \mathbb{P}[u|G,A] \\
        &= \mathbb{P}[v|G,A]
    \end{align*}
    The second equation holds because the events of reaching $v$ from different parents are exclusive.
\end{itemize}
\end{proof}

\begin{algorithm}[tb]
\caption{Top-Down Traversal of MDD}\label{alg:forward}
\textbf{Input} An MDD $G(V, E)$ representation of constraint $c$, a real assignment $P\in\tilde{D}$.\\
\textbf{Output} Circuit-output probability ${COP}_c(P)$.

\begin{algorithmic}[1]
\STATE Let $m_{td}[v]$ be the forward message of each node $v\in{}V$, where the root node is initialized as 1, otherwise 0.\\
\FOR{$v\in{}\texttt{sorted}(V)$}
    \FOR{$u = \mathrm{child}(v)$}
        \STATE $m_{td}[u] += p_{\mathrm{id}(v),\mathrm{label}(v,u)} \cdot m_{td}[v]$ 
    \ENDFOR
\ENDFOR
\STATE \textbf{return} $m_{td}[true]$
\end{algorithmic}
\end{algorithm}

Alternatively, traversing the MDD bottom-up can also compute the circuit-output probability. 

\begin{lemma}[Bottom-Up Evaluation]\label{lmm:backward}
    Let $G(V,E)$ be the MDD of constraint $c$, run Algorithm~\ref{alg:forward} on $G$ with a real assignment $P\in\tilde{D}$.
    For each $v\in V$, let $f_{c,v}$ be the sub-function of the sub-MDD generated by regarding $v$ as the root.
    The following holds:
    \begin{equation*}
        m_{bu}[v] = {COP}_{c,v}(P),
    \end{equation*}
    \emph{i.e.}, Algorithm~\ref{alg:backward} returns the COP of $G$ and in time $O(|V|)$.
\end{lemma}

\begin{proof}[Proof of Lemma~\ref{lmm:backward}]
We prove this by structural induction on MDD $G(V, E)$.
Let $A\in\tilde{D}$, we have
\begin{itemize}
    \item \emph{Basic step}: For the terminal nodes, $m_{bu}[\texttt{true}]=1$ and $m_{bu}[\texttt{false}]=0$.
    The statement holds since the sub-functions given by terminal nodes are constant functions with values 1 and 0, respectively.

    \item \emph{Inductive step}: For each non-terminal node $v$, let $v.i$ be the child node connecting $v$ by edge label $i$.
    after Algorithm~\ref{alg:backward} terminates, we have
    \begin{equation*}
        m_{bu}[v] = \sum_{i\in\mathrm{dom}(x_{id(v)})} a_{id(v), i}\cdot m_{bu}[v.i].
    \end{equation*}

    By the inductive hypothesis, we have
    \begin{align*}
        m_{bu}[v] 
        &= \sum_{i=0}^k a_{id(v), i}\cdot {COP}_{c,v.i}(A)\\
        &= \sum_{i=0}^k \mathbb{P}[b_{id(v)}=i]\cdot {COP}_{c,v.i}(A)\\
        &= {COP}_{c, v}(A)
    \end{align*}
    where ${COP}_{c, v}$ denotes the circuit-output probability of subgraph of $G$ rooted at node $v$.
\end{itemize}
\end{proof}

\begin{algorithm}[tb]
\caption{Bottom-Up Traversal of MDD}\label{alg:backward}
\textbf{Input} An MDD $G(V, E)$ representation of constraint $c$, a real assignment $P\in\tilde{D}$.\\
\textbf{Output} Circuit-output probability ${COP}_c(P)$.

\begin{algorithmic}[1]
\STATE Let $m_{bu}[v]$ be the bottom-up message of each node $v\in{}V$, where the true terminal node is initialized as 1, otherwise 0.\\
\FOR{$v\in{}\texttt{sorted}(V, \texttt{reverse=True})$}
    \FOR{$u \in \mathrm{child}(v)$}
        \STATE $i=\mathrm{id}(v)$
        \STATE $m_{bu}[v] += p_{\mathrm{id}(v),\mathrm{label}(v,u)} \cdot m_{bu}[u]$ 
    \ENDFOR
\ENDFOR
\STATE \textbf{return} $m_{bu}[root]$
\end{algorithmic}
\end{algorithm}

\subsection{Differentiation}
The main idea of differentiation is from probabilistic inference~\cite{pearl1988belief,shafer1990probability}, which traverses the graph model top-down and bottom-up. 
The following theorem shows that the partial derivative of COP can be obtained by top-down and bottom-up messages from Algorithm~\ref{alg:forward} and~\ref{alg:backward}.

\begin{theorem}[Derivatives]\label{thm:diff}
    After running Algorithm~\ref{alg:forward} and~\ref{alg:backward} on MDD for constraint $c$ with a real assignment $P\in\tilde{D}$, the partial derivative of COP on $p_{i,j}$ for all $j\in\mathrm{dom}(x_i)$ can be computed by
    \begin{equation}
        \frac{\partial{COP}_c(P)}{\partial p_{i,j}} = \sum_{v\in\{u:id(u)=i\}}m_{td}[v] \cdot m_{bu}[v.j]\label{eq:diff}
    \end{equation}
\end{theorem}

\begin{proof}[Proof of Theorem~\ref{thm:diff}]
Given an {MDD} $G(V, E)$, expand all terminal nodes such that each can only have one parent node.
The expanded {MDD} is called $\mathcal{G}$, which has a true terminal set $T$ and a false terminal set $F$.
We have the following by repeating the proof of Lemma~\ref{lmm:forward} and~\ref{lmm:backward} on $\mathcal{G}$.
\begin{align*}
    {COP}_c(A) = \sum_{v\in{}T} \mathbb{P}[v|A,\mathcal{G}] = \sum_{v\in{}T} m_{td}'(v) 
\end{align*}
    
For any non-terminal node $v\notin(T\cup F)$, 
\begin{align*}
    m_{td}[v] &= m_{td}'[v], \\
    m_{bu}[v] &= m_{bu}'[v].
\end{align*}
    
For any true terminal node $v\in{T}$ and false terminal node $u\in{F}$, 
\begin{align*}
    m_{bu}'[v] &= 1, \\
    m_{bu}'[u] &= 0.
\end{align*}

Now, we are ready to prove Theorem~\ref{thm:diff}.
\begin{align*}
    {COP}_c(A) 
    &= \mathbb{P}[f_c=1|A, \mathcal{G}] \\
    &= \sum_{v\in\{u:id(u)=i\}} \mathbb{P}[v|A, \mathcal{G}] \cdot \mathbb{P}[f_c=1|v, A, \mathcal{G}] \\
    &= \sum_{v\in\{u:id(u)=i\}} m_{td}'[v] \cdot m_{bu}'[v] \\
    &= \sum_{v\in\{u:id(u)=i\}} m_{td}[v] \cdot m_{bu}[v] \\
    &= \sum_{v\in\{u:id(u)=i\}} m_{td}[v] \cdot \sum_{j=0}^k a_{i,j}m_{bu}[v.j]
\end{align*}

Hence we have
\begin{align*}
    \frac{\partial{COP}_c(P)}{\partial p_{i,j}} = \sum_{v\in\{u:id(u)=i\}} m_{td}[v] \cdot m_{bu}[v.j]
\end{align*}
\end{proof}

\subsection{Parallel Decision Diagram Traversal}

The decision diagram can be traversed in parallel.
For example, MDDs can be represented as Listing~\ref{lst:c13} and Listing~\ref{lst:c15}. 

\begin{lstlisting}[language=python, caption={$x_1\wedge x_2 \wedge x_3 \wedge x_4$}, label={lst:c13}]
# vid, nid, eid, cid
1 1 0 5
1 1 1 2
1 1 2 5
2 2 0 5
2 2 1 3
2 2 2 5
3 3 0 5
3 3 1 4
3 3 2 5
4 4 0 5
4 4 1 6
4 4 2 5
0 0 0 0
0 0 0 0
0 0 0 0
\end{lstlisting}

\begin{lstlisting}[language=python, caption={$x_1\wedge (x_2 < x_3) \wedge x_4$}, label={lst:c15}]
# vid, nid, eid, cid
1 1 0 6
1 1 1 2
1 1 2 6
2 2 0 3
2 2 1 4
2 2 2 6
3 3 0 6
3 3 1 5
3 3 2 5
3 4 0 6
3 4 1 6
3 4 2 5
4 5 0 6
4 5 1 7
4 5 2 6
\end{lstlisting}

For every MDD, we call the kernel functions to traverse one edge at a time. 
The corresponding top-down and bottom-up message propagation kernels are shown in Listing~\ref{lst:cuda_forward} and Listing~\ref{lst:cuda_backward}.
To ensure uniform execution across threads, shorter MDD traces are padded with zeros.
The kernel functions are designed to detect such padded entries (e.g., when \texttt{nid} equals zero) and skip message propagation accordingly.

\begin{lstlisting}[language=C++, caption={CUDA-style top-down message propagation kernel function}, label={lst:cuda_forward}]
__global__ void top_down_kernel(
    const int* variable_id, const int* node_id, const int* edge_id, const int* child_id,
    const float* p, const int* pointer_id, float* m_td
) {
    int idx = blockIdx.x * blockDim.x + threadIdx.x;
    if (idx >= num_mdd) return;

    int vid = variable_id[idx];
    if (vid == 0) return;
    int nid = node_id[idx];
    int eid = edge_id[idx];
    int cid = child_id[idx];
    int pid = pointer_id[vid];

    atomicAdd(&m_td[cid], m_td[nid] * p[pid + eid]);
}
\end{lstlisting}

\begin{lstlisting}[language=C++, caption={CUDA-style bottom-up message propagation kernel function}, label={lst:cuda_backward}]
__global__ void bottom_up_kernel(
    const int* variable_id, const int* node_id, const int* edge_id, const int* child_id,
    const float* p, const int* pointer_id, const int num_mdd, float* m_bu 
) {
    int idx = blockIdx.x * blockDim.x + threadIdx.x;
    if (idx >= num_mdd) return;

    int vid = variable_id[idx];
    if (vid == 0) return;
    int nid = node_id[idx];
    int eid = edge_id[idx];
    int cid = child_id[idx];
    int pid = pointer_id[vid];

    atomicAdd(&m_bu[nid], m_bu[cid] * p[pid + eid]);
}
\end{lstlisting}

\end{document}